\pdfoutput=1
\documentclass[letterpaper,11pt]{article}
\usepackage[paperwidth=8.5in,paperheight=11.0in,left=1in, right=1in, top=1.25in, bottom=1.25in, dvips]{geometry}
\usepackage{lmodern}
\usepackage{anyfontsize}
\usepackage{graphicx,float,psfrag,amsmath,amsthm,amssymb,nicefrac,mathtools}
\usepackage[english]{babel}
\usepackage{listings}
\usepackage{soul}
\usepackage{multirow}
\usepackage{algorithm,algpseudocode}
\usepackage[numbers]{natbib}

\usepackage{enumitem}
\usepackage{url,color,verbatim}

\theoremstyle{plain}

\newtheorem{theorem}{Theorem}

\newtheorem{lemma}{Lemma}

\newtheorem*{prop*}{Proposition}

\theoremstyle{definition}

\newtheorem*{defn*}{Definition}

\newtheorem*{exmp*}{Example}

\newtheorem*{conj*}{Conjecture}

\theoremstyle{remark}

\newtheorem*{rmk*}{Remark}

\DeclareMathOperator*{\argmax}{arg\,max}

\def \ifempty#1{\def\temp{#1} \ifx\temp\empty }

\newcommand{\N}{\ensuremath{\mathbf{N}}}

\newcommand{\Q}{\ensuremath{\mathbf{Q}}}
\newcommand{\RR}{\ensuremath{\mathbb{R}}}

\newcommand{\q}{\ensuremath{\mathbf{q}}}

\newcommand{\y}{\ensuremath{\mathbf{y}}}

{%
\begin{list}{#1}{
\vspace{-\topsep}
\vspace{-\partopsep}
\setlength{\itemindent}{0cm}
\setlength{\rightmargin}{0cm}
\setlength{\listparindent}{0cm}
\settowidth{\labelwidth}{#1}
\setlength{\leftmargin}{\labelwidth}
\addtolength{\leftmargin}{\labelsep}
\setlength{\itemsep}{0cm}
}%
}%
{%
\end{list}
\vspace{-\topsep}
\vspace{-\partopsep}
}

{\begin{enumerate}%
}%
{\end{enumerate}}

\renewcommand{\Re}{\mathbb{R}}
\usepackage{multirow}

\def\A{{\cal A}}

\def\Y{{\cal Y}}

\def\D{{\cal D}}
\def\L{{\cal L}}
\def\Q{{\cal Q}}
\def\S{{\cal S}}
\def\O{{\cal O}}
\def\s{\mathbf{s}}
\def\y{\mathbf{y}}
\def\q{\mathbf{q}}
\def\r{\mathbf{r}}

\newcommand{\caja}[4][1]{{%
		\renewcommand{\arraystretch}{#1}%
		\begin{tabular}[#2]{@{}#3@{}}%
			#4%
		\end{tabular}%
}}

\makeindex             

\title{Quasi-Newton Optimization Methods For\\Deep Learning Applications}
\author{Jacob Rafati \\ 
	Electrical Engineering and Computer Science, Univeristy of California, Merced\\
	\url{http://rafati.net}
	\and
	Roummel F. Marcia\\
	Department of Applied Mathematics, Univeristy of California, Merced\\
	\url{http://faculty.ucmerced.edu/rmarcia/}}

\date{}
\begin{document}
\maketitle

\section*{Abstract}
Deep learning algorithms often require solving a highly non-linear and nonconvex unconstrained optimization problem. Methods for solving optimization problems in large-scale machine learning, such as deep learning and deep reinforcement learning (RL), are generally restricted to the class of first-order algorithms, like stochastic gradient descent (SGD). While SGD iterates are inexpensive to compute, they have slow theoretical convergence rates. Furthermore, they require exhaustive trial-and-error to fine-tune many learning parameters. Using second-order curvature information to find search directions can help with more robust convergence for non-convex optimization problems. However, computing Hessian matrices for large-scale problems is not computationally practical. Alternatively, quasi-Newton methods construct an approximate of the Hessian matrix to build a quadratic model of the objective function. Quasi-Newton methods, like SGD, require only first-order gradient information, but they can result in superlinear convergence, which makes them attractive alternatives to SGD. The limited-memory Broyden-Fletcher-Goldfarb-Shanno (L-BFGS) approach is one of the most popular quasi-Newton methods that construct positive definite Hessian approximations. In this chapter, we propose efficient optimization methods based on L-BFGS quasi-Newton methods using line search and trust-region strategies. Our methods bridge the disparity between first- and second-order methods by using gradient information to calculate low-rank updates to Hessian approximations. We provide formal convergence analysis of these methods as well as empirical results on deep learning applications, such as image classification tasks and deep reinforcement learning on a set of ATARI 2600 video games. Our results show a robust convergence with preferred generalization characteristics as well as fast training time.

\section{Introduction}
\label{sec:intro}
Deep learning (DL) is becoming the leading technique for solving large-scale machine learning (ML) problems, including image classification, natural language processing, and large-scale regression tasks \citep{Goodfellow-et-al:2016:Deep-Learning-Book,Wani2019}. Deep learning algorithms attempt to train a function approximation (\emph{model}), usually a deep convolutional neural network (CNN), over a large dataset. In most of deep learning and deep reinforcement learning (RL) algorithms, solving an \emph{empirical risk minimization} (ERM) problem is required \citep{Hastie:2009:Elements-of-statistical-learning-book}. ERM is a highly nonlinear and nonconvex unconstrained optimization of the form 
\begin{align}
\min_{w \in \RR^n} \L(w) \triangleq \frac{1}{N} \sum_{i=1}^{N} \ell_i(w),
\label{eq:unconstrained-optimization}
\end{align}
where $w \in \RR^n$ is the vector of trainable parameters of the CNN model, $n$ is the number of the learning parameters, $N$ is the number of observations in a training dataset, and $\ell_i(w)\triangleq \ell(w;x_i,y_i)$ is the error of the current model's prediction for the $i$th observation of the \emph{training} dataset, $\D=\{(x_i,y_i)~|~i=1,\dots,N\}$. 

\subsection{Existing Methods}
Finding an efficient optimization algorithm for the large-scale, nonconvex ERM problem \eqref{eq:unconstrained-optimization} has attracted many researchers \citep{Goodfellow-et-al:2016:Deep-Learning-Book}. There are various algorithms proposed in the machine learning and optimization literatures to solve \eqref{eq:unconstrained-optimization}. Among those, one can name first-order methods such as stochastic gradient descent (SGD) methods \citep{Robbins:1951:SGD,bottou2010large,Duchi2011,recht2011hogwild}, the quasi-Newton methods \citep{adhikari2017limited,Le:2011:BFGS-and-CG,Erway-etal-2018-arXiv,Xu:2017:empirical}, and also Hessian-free methods \citep{Ma10,Ma11,Ma12,Bollapragada:2016:Hessian-Free}.

Since, in large-scale machine learning problems usually, $N$ and $n$ are very large numbers, the computation of the true gradient $\nabla \L(w)$ is expensive and the computation of the true Hessian $\nabla^2 \L(w)$ is not practical. Hence, most of the optimization algorithms in machine learning and deep learning literature are restricted to variants of first-order gradient descent methods, such as SGD methods. SGD methods use a small random sample of data, $J_k \in \S$, to compute an approximate of the gradient of the objective function, $\nabla \L^{(J_k)}(w) \approx \nabla \L(w)$.  At each iteration of the learning update, the parameters are updated as  $w_{k+1} \gets w_{k} - \eta_k \nabla \L^{(J_k)}(w_k)$, where $\eta_k$ is referred to as the \emph{learning rate}. 

The computational cost-per-iteration of SGD algorithms is small, making them the most widely used optimization method for the vast majority of deep learning applications. However, these methods require fine-tuning of many hyperparameters, including the learning rates. The learning rate is usually chosen to be very small; therefore, the SGD algorithms require revisiting many epochs of data during the learning process. Indeed, it is unlikely that the SGD methods perform successfully in the first attempt at a problem, though there is recent work that addresses tuning  hyperparameters automatically 
(see e.g., \citep{Zeiler2012,kingma2014adam}).

Another major drawback of SGD methods is that they struggle with
saddle-points that occur in most nonconvex optimization problems. These saddle-points have an undesirable effect on the model's generalization ability. On the other hand, using the second-order curvature information, can help produce more robust convergence. Newton's method, which is a second-order method, uses the Hessian, $\nabla^2 \L(w)$, and the gradient to find the search direction, $p_k = -\nabla^2 \L(w_k)^{-1} \nabla \L(w_k)$.  A line-search strategy is often used to find the step length along the search direction to guarantee convergence. The main bottleneck in second-order methods is the serious computational challenges involved in the computation of the Hessian, $\nabla^2 \L(w)$, for large-scale problems, in which it is not practical because $n$ is large. Quasi-Newton methods and Hessian-free methods both use approaches to approximate the Hessian matrix without computing and storing the true Hessian matrix. Specifically, Hessian-free methods attempt to find an approximate Newton direction by solving $\nabla^2 \L(w_k) p_k = -\nabla \L(w_k)$ without forming the Hessian using conjugate-gradient methods \citep{Ma10,Ma11,Ma12,Bollapragada:2016:Hessian-Free}.  

Quasi-Newton methods form an alternative class of first-order methods for solving the large-scale nonconvex optimization problem in deep learning. These methods, as in SGD, require only computing the first-order gradient of the objective function. By measuring and storing the difference between consecutive gradients, quasi-Newton methods construct \emph{quasi-Newton matrices} 
$\{B_k\}$ which are low-rank updates to the previous Hessian approximations for 
estimating $\nabla^2 \L(w_k)$ at each iteration. They build a quadratic model of the objective function by using these quasi-Newton matrices and use that model to find a sequence of search directions that can result in superlinear convergence. Since these methods do not require the second-order derivatives, they are more efficient than Newton's method for large-scale optimization problems \citep{Nocedal-Wright:2006:Numerical-Optimization-Book}. 

There are various quasi-Newton methods proposed in the literature. They differ in how they define and construct the quasi-Newton matrices $\{B_k\}$, how the search directions are computed, and how the parameters of the model are updated \citep{Nocedal-Wright:2006:Numerical-Optimization-Book,Brust-etal:2017:DenseInit,brust2017solving,Le:2011:BFGS-and-CG}.

\subsection{Motivation}
The Broyden-Fletcher-Goldfarb-Shanno (BFGS) method  \citep{Bro70,Fle70,Gol70,Sha70} is considered the most widely used quasi-Newton algorithm, which produces a positive-definite matrix $B_k$ for each iteration. The conventional BFGS minimization employs \emph{line-search}, which first attempts to find the search directions by computing $p_k = - B^{-1}_{k} \nabla \L(w_k) $ and then decides on the step size $\alpha_k \in (0,1]$ based on sufficient decrease and curvature conditions \citep{Nocedal-Wright:2006:Numerical-Optimization-Book} for each iteration $k$ and then update the parameters $w_{k+1} = w_k + \alpha_k p_k$. The line-search algorithm first tries the unit step length $\alpha_k = 1$ and if it does not satisfy the sufficient decrease and the curvature conditions, it recursively reduces $\alpha_k$ until some stopping criteria (for example $\alpha_k < 0.1$). 

Solving $B_{k} p_k = -\nabla \L(w_k) $ can become computationally expensive when $B_k$ becomes a high-rank update. The \emph{limited-memory} BFGS (L-BFGS) method  constructs a sequence of  low-rank updates to the Hessian  approximation;  consequently solving $p_k = B_k^{-1}\nabla \L(w_k) $ can be done efficiently. As an alternative to gradient descent, limited-memory quasi-Newton algorithms with line search have been implemented in a deep learning setting \citep{le2011optimization}. These methods approximate second derivative information, improving the quality of each training iteration and circumventing the need for application specific parameter tuning.

There are computational costs associated with the satisfaction of the sufficient decrease and curvature conditions, as well as finding $\alpha_k$ using line-search methods. Also if the curvature condition is not satisfied for $\alpha_k \in (0,1]$, the L-BFGS matrix may not stay positive definite, and the update will become unstable. On the other hand, if the search direction is rejected in order to preserve the positive definiteness of L-BFGS matrices, the progress of learning might stop or become very slow.             

Trust-region methods attempt to find the search direction, $p_k$, in a region within which they trust the accuracy of the quadratic model of the objective function, $\Q(p_k)\triangleq \frac{1}{2} p_k^T B_k p_k + \nabla \L(w_k)^T p_k$. These methods not only have the benefit of being independent from the fine-tuning of hyperaparameters, but they may improve upon the training performance and the convergence robustness of the line-search methods. Furthermore, trust-region L-BFGS methods 
can easily reject the search directions if the curvature condition is not satisfied in order to preserve the positive definiteness of the L-BFGS matrices \citep{Rafati-et-al:2018:EUSIPCO}. The computational bottleneck of trust-region methods is the solution of the trust-region subproblem. However, recent work has shown that the trust-region subproblem can be efficiently solved if the Hessian approximation, $B_k$, is chosen to be a quasi-Newton matrix \citep{brust2017solving,BurdakovLMTR16}.

\subsection{Applications and Objectives}
Deep learning algorithms attempt to solve large-scale machine learning problems by learning a model (or a parameterized function approximator) from the observed data in order to predict the unseen events. The model used in deep learning is an artificial neural network which is a stack of many convolutional layers, fully connected layers, nonlinear activation functions, etc. Many data-driven or goal-driven applications can be approached by deep learning methods. Depending on the application and the data, one should choose a proper architecture for the model and define an empirical loss function. (For the state-of-the-art deep neural network architectures, and deep learning applications for supervised learning and unsupervised learning,  see \citep{Wani2019}.) The common block in all deep learning algorithms is the optimization step for solving the ERM problem defined in \eqref{eq:unconstrained-optimization}.  

In this chapter, we present methods based on quasi-Newton optimization for solving the ERM problem for deep learning applications. For numerical experiments, we focus on two deep learning applications, one in supervised learning and the other one in reinforcement learning. The proposed methods are general purpose and can be employed for solving optimization steps of other deep learning applications. 

First, we introduce novel large-scale L-BFGS optimization methods using the trust-region strategy -- as an alternative to the gradient descent methods. This method is called Trust-Region Minimization Algorithm for Training Responses (TRMinATR) \citep{Rafati-et-al:2018:EUSIPCO}. We implement practical computational algorithms to solve the Empirical Risk Minimization (ERM) problems that arise in machine learning and deep learning applications. We provide empirical results on the classification task of the MNIST dataset and show robust convergence with preferred generalization characteristics. Based on the empirical results, we provide a comparison between the trust-region strategy with the line-search strategy on their different convergence properties. TRMinATR solves the associated trust-region subproblem, which can be computationally intensive in large scale problems, by efficiently computing a closed form solution at each iteration. Based on the distinguishing characteristics of trust-region algorithms, unlike line-search methods, the progress of the learning will not stop or slow down due to the occasional rejection of the undesired search directions. We also study techniques for initialization of the positive definite L-BFGS quasi-Newton matrices in the trust-region strategy so that they do not introduce any false curvature conditions when constructing the quadratic model of the objective function. 

Next, we investigate the utility of quasi-Newton optimization methods in deep reinforcement learning (RL) applications. RL -- a class of machine learning problems -- is learning how to map situations to actions so as to maximize numerical reward signals received during the experiences that an artificial agent has as it interacts with its environment \citep{RL-Book:Sutton:Barto:2017}. An RL agent must be able to sense the state of its environment and must be able to take actions that affect the state. The agent may also be seen as having a goal (or goals) related to the state of the environment. One of the challenges that arise in real-world reinforcement learning (RL) problems is the ``curse of dimensionality''. Nonlinear function approximators coupled with RL have made it possible to learn abstractions over high dimensional state spaces \citep{SuttonRS:1996:Coarse,Rafati-Noelle:2015:CSC,Rafati-Noelle:2017:CCCN,Rafati2019phd,Rafati-Noelle:2018:HRL-arXiv,Melo:2008:Analysis-Q-learnings}.
Successful examples of using neural networks for RL include learning how to play the game of
Backgammon at the Grand Master
level~\citep{TesauroG:1995:TDGammon}. More recently, researchers at DeepMind Technologies used a deep Q-learning algorithm to play various ATARI games from the raw screen image stream \citep{DeepMind:Atari:2013,DeepMind:Nature:2015}. The Deep Q-learning algorithm \citep{DeepMind:Atari:2013} employed a convolutional neural network (CNN) as the state-action value function approximation. The resulting performance on these games was frequently at or better than the human level. In another effort, DeepMind used deep CNNs and a Monte Carlo Tree Search algorithm that combines supervised learning and RL to learn how to play the game of Go at a super-human level \citep{DeepMind-AlphaGo}. 

We implement an L-BFGS optimization method  for deep reinforcement learning framework. Our deep L-BFGS Q-learning method is designed to be efficient for parallel computation using GPUs. We investigate our algorithm using a subset of the ATARI 2600 games, assessing its ability to learn robust representations of the state-action value function, as well as its computation and memory efficiency. We also analyze the convergence properties of Q-learning using a deep neural network employing L-BFGS optimization.

\subsection{Chapter Outline}
In Section \ref{sec:prelim}, we introduce a brief background on machine learning, deep learning, and optimality conditions for unconstrained optimization. In Section \ref{sec:optim}, we introduce two common optimization strategies for unconstrained optimization, i.e. line-search and trust-region. In Section \ref{sec:qn}, we introduce quasi-Newton methods based on L-BFGS optimization in both line-search and trust-region strategies. In Section \ref{sec:image-rec}, we implement algorithms based on trust-region and line-search L-BFGS for image recognition task. In Section \ref{sec:rl} we introduce the RL problem and methods based on L-BFGS line-search optimization for solving ERM problem in deep RL applications. 

\section{Unconstrained Optimization Problem}
\label{sec:prelim}

In the unconstrained optimization problem, we want to solve the minimization problem \eqref{eq:unconstrained-optimization}.
\begin{align}
\min_{w} \L(w),
\end{align}
where $\L:~\RR^n \rightarrow \RR$ is a smooth function. A point $w^*$ is a \emph{global} minimizer if $\L(w^*) \leq \L(w)$ for all $w \in \RR^n$. Usually $\L$ is a nonconvex function and most algorithms are only able to find the \emph{local} minimizer. A point $w^*$ is a local minimizer if there is a neighborhood $\N$ of $w^*$ such that $\L(w^*) \leq \L(w)$ for all $w \in \N$. For convex functions, every local minimizer is also a global minimizer, but this statement is not valid for nonconvex functions. If $\L$ is twice continuously differentiable we may be able to tell that $w^*$ is a local minimizer by examining the gradient $\nabla \L(w^*)$ and the Hessian $\nabla^2 \L(w^*)$. Let's assume that the objective function, $\L$, is smooth: the first derivative (gradient) is differentiable and the second derivative (Hessian) is continuous. To study the minimizers of a smooth function, Taylor's theorem is essential. 
\begin{theorem}[Taylor's Theorem]
	Suppose that $\L:~\RR^n \rightarrow \RR$ is continuously differentiable. Consider $p \in \RR^n$ such that $\L(w+p)$ is well defined, then we have
	\begin{align}
	\L(w+p) = \L(w) + \nabla \L(w+tp)^T p, \quad \textrm{for some } t \in (0,1).
	\end{align}
	Also if $\L$ is twice continuously differentiable,
	\begin{align}
	\L(w+p) = \L(w) + \nabla \L(w+tp)^T p + \frac{1}{2}p^T \nabla^2 \L(w+tp)p,\textrm{ for some } t \in (0,1).
	\label{eq:taylor}
	\end{align} 
\end{theorem}
A point $w^*$ is a local minimizer of $\L$ only if $\nabla \L(w^*) = 0$.  This is knowns as the \emph{first-order optimality condition.}  
In addition, if $\nabla^2 \L(w^*)$ is positive definite, then
$w^*$ is guaranteed to be a local minimizer. This is known as \emph{second-order sufficient condition} \cite{Nocedal-Wright:2006:Numerical-Optimization-Book}.  

%

\section{Optimization Strategies}
\label{sec:optim}
In this section, we briefly introduce two optimization strategies that are commonly used, namely \emph{line search} and \emph{trust-region} methods \citep{Nocedal-Wright:2006:Numerical-Optimization-Book}. Both methods seek to minimize the objective function $\L(w)$ in \eqref{eq:unconstrained-optimization} by defining a sequence of iterates $\{w_k\}$ which are governed by the search direction $p_k$. Each respective method is defined by its approach to computing the search direction $p_k$ so as to minimize the quadratic model of the objective function defined by
\begin{eqnarray}\label{eq:quadratic-model}
\Q_k(p) \triangleq g_k^T p + \frac{1}{2} p^T B_k p, 
\end{eqnarray}
where $g_k\triangleq\nabla\L(w_k)$ and $B_k$ is an approximation to the Hessian matrix $\nabla^2\L(w_k)$. Note that $\Q_k(p)$ is a quadratic approximation of $\L(w_k + p) - \L(w_k)$ based on the Taylor's expansion in \eqref{eq:taylor}.
\subsection{Line Search Methods} 
Each iteration of a line search method computes a search direction $p_k$ by minimizing a quadratic model of the objective function,
\begin{align}
p_k = \arg \min_{p \in \mathbb{R}^n} \Q_k(p) \triangleq \frac{1}{2}p^T B_k p + g_k^T p,
\label{eqn:line-search-subproblem}
\end{align}
and then decides how far to move along that direction. The iteration is given by $w_{k+1} = w_k + \alpha_k p_k$, where $\alpha_k$ is called the step size. If $B_k$ is a positive definite matrix, the minimizer of the quadratic function can be found as $p_k = - B_k^{-1} g_k$. 
The ideal choice for step size $\alpha_k>0$ is the global minimizer of the univariate function $\phi(\alpha)=\L(w_k + \alpha p_k)$, but in practice $\alpha_k$ is chosen to satisfy sufficient decrease and curvature conditions, e.g., 
the Wolfe conditions \citep{Wolfe1969,Nocedal-Wright:2006:Numerical-Optimization-Book} given by
\begin{subequations}
	\begin{align}
	&\L(w_k + \alpha_k p_k)  \leq \L(w_k) + c_1 \alpha_k \nabla \L(w_k)^T p_k, \\
	&\nabla \L(w_k + \alpha_k p_k)^T p_k \geq c_2 \nabla \L(w_k)^T p_k,
	\end{align}
	\label{eqn:Wolfe-Conditions}
\end{subequations}
with $0 < c_1 < c_2 < 1$.

The general pseudo-code for the line search method is given in Algorithm \ref{Algo:line-search-pseudo-code} (see \cite{Nocedal-Wright:2006:Numerical-Optimization-Book} for details).

\begin{algorithm}[t!]
	\centering
	\begin{algorithmic}
		\State {\bf Input: } $w_0$, tolerance $\epsilon > 0$ 
		\State $k \gets 0$
		\Repeat
		\State Compute $g_k = \nabla \L(w_k)$
		\State Calculate $B_k$ 
		\State Compute search direction $p_k$ by solving \eqref{eqn:line-search-subproblem}
		\State find $\alpha_k$ that satisfies Wolfe Conditions in \eqref{eqn:Wolfe-Conditions}
		\State $k \gets k+1$
		\Until{$\| g_k \| < \epsilon$ or $k$ reached to max number of iterations}				
	\end{algorithmic}
	\caption{Line Search Method pseudo-code.}
	\label{Algo:line-search-pseudo-code}
\end{algorithm}

\subsection{Trust-Region Methods}
Trust-region methods generate a sequence of iterates $w_{k+1} = w_k + p_k$,
where each search step, $p_k$, is obtained by
solving the following trust-region subproblem:
\begin{align}
p_k = \underset{p \in \mathbb{R}^n}{\text{argmin}}~\Q_k(p)\triangleq \frac{1}{2}p^T B_k p + g_k^T p, \text{ such that \ } \| p \|_2  \leq \delta_k,
\label{eq:TR-subproblem}
\end{align} 
where $\delta_k>0$ is the trust-region radius. The global solution to the trust-region subproblem \eqref{eq:TR-subproblem}
can be characterized by the optimality conditions given in the 
following theorem due to \cite{gay1981computing} and \cite{more1983computing}.

\begin{theorem} 
	Let $\delta_k$ be a positive constant. A vector $p^*$ is a global solution of the trust-region subproblem \eqref{eq:TR-subproblem} if and only if $\|p^*\|_2\leq \delta_k$ and there exists a unique $\sigma^*\geq 0$ such that $B+\sigma^*I$ is positive semidefinite and 
	\begin{align}\label{eq:optcond}
	(B+\sigma^*I) p^*=-g\quad \text{and}\quad \sigma^*(\delta-\|p^*\|_2)=0.
	\end{align}
	Moreover, if $B+\sigma^*I$ is positive definite, then the global minimizer is unique. 
\end{theorem}

\begin{figure}[hbt!]
	\centering
	\includegraphics[width=.6\textwidth]{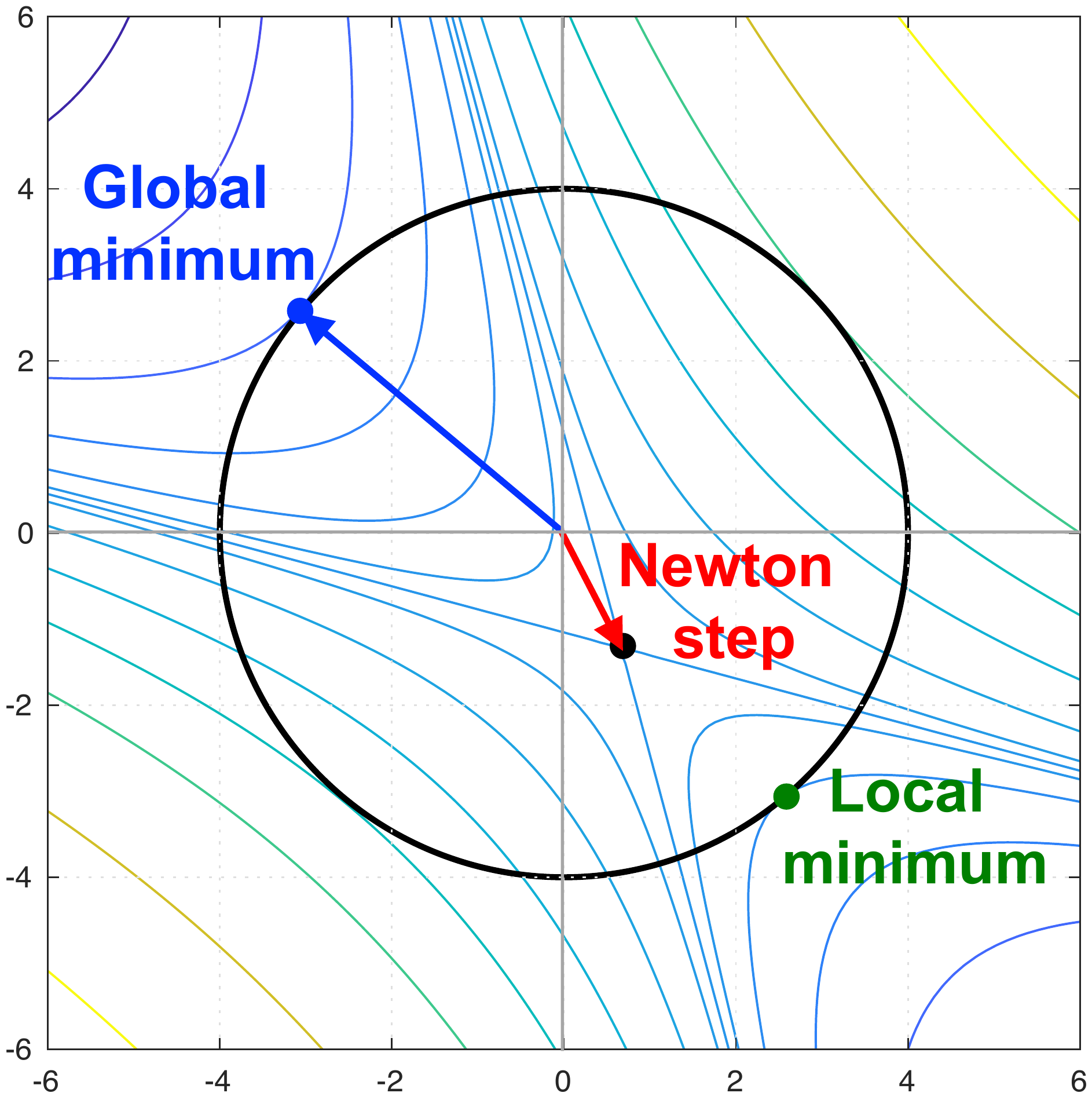}
	\caption{An illustration of trust-region methods.  For indefinite matrices,
		the Newton step (in red) leads to a saddle point.  The global minimizer
		(in blue) is characterized by the conditions in Eq.\ \eqref{eq:optcond}
		with $B + \sigma^*I$ positive semidefinite.  In contrast, local minimizers
		(in green) satisfy Eq.\ \eqref{eq:optcond} with $B+\sigma^*I$ not positive semidefinite.}
\end{figure}

The general pseudo-code for the trust region method is given in Algorithm \ref{Algo:trust-region-pseudo-code}. 
See Algorithm 6.2 of \cite{Nocedal-Wright:2006:Numerical-Optimization-Book} for details. For further details on trust-region methods, see \cite{ConGT00a}.

\begin{algorithm}[t!]
	\centering
	\begin{algorithmic}
		\State \textbf{Input: } $w_0$, $\epsilon >0$, $\hat{\delta} > 0$, $\delta_0 \in (0,\hat{\delta})$, $\eta \in [0,1/4)$
		\State $k \gets 0$
		\Repeat
		\State Compute $g_k = \nabla \L(w_k)$
		\State Construct quasi-Newton matrix $B_k$ 
		\State Compute search direction $p_k$ by solving \eqref{eq:TR-subproblem}
		\State $ared \gets \L(w_k) - \L(w_k + p_k)$
		\State $pred \gets -\Q_k(p_k)$
		\State $\rho_k \gets  ared/pred$
		\State Update trust-region radius $\delta_k$ 
		\If{$\rho_k > \eta$}
		\State $w_{k+1} = w_k + p_k$
		\Else
		\State $w_{k+1} = w_k$
		\EndIf			
		\State $k \gets k+1$
		\Until{$\| g_k \| < \epsilon$}				
	\end{algorithmic}
	\caption{Trust region method pseudo-code.}
	\label{Algo:trust-region-pseudo-code}
\end{algorithm}

\section{Quasi-Newton Optimization Methods}
\label{sec:qn}
Methods that use $B_k = \nabla^2 \L(w_k)$ for the Hessian 
in the quadratic model in \eqref{eq:quadratic-model}
typically exhibit quadratic rates of convergence.  However, in large-scale problems (where $n$ and $N$ are both large), computing the true Hessian explicitly is not practical. In this case,
quasi-Newton methods are viable alternatives because they exhibit super-linear convergence rates while maintaining memory and computational efficiency. Instead of the true Hessian, quasi-Newton methods use an approximation, $B_k$, which is updated after each step to take into account the additional knowledge gained during the step.

Quasi-Newton methods, like gradient descent methods, require only the computation of first-derivative information. They can construct a model of the objective function by measuring the changes in the consecutive gradients for estimating the Hessian. Most methods store the displacement, $\s_k \triangleq  w_{k+1} - w_k $, and the change of gradients, $\y_k \triangleq \nabla \L(w_{k+1}) - \nabla \L(w_{k})$, to construct the Hessian  approximations, $\{B_k\}$. The quasi-Newton matrices are required to satisfy the secant equation, $B_{k+1}  \s_k = \y_k$. Typically, there are additional conditions imposed on $B_{k+1}$, such as symmetry (since the exact Hessian is symmetric), and a requirement that the update to obtain $B_{k+1}$ from $B_k$ is low rank, meaning that the Hessian approximations cannot change too much from one iteration to the next.  Quasi-Newton methods vary in how this update is defined. The matrices are defined recursively with the initial matrix, $B_0$, taken to be $B_0 = \lambda_{k+1} I$, where the scalar $\lambda_{k+1} > 0$.
\subsection{The BFGS Update}
Perhaps the most well-known among all of the quasi-Newton methods is the
Broyden-Fletcher-Goldfarb-Shanno (BFGS) update \citep{liu1989limited,Nocedal-Wright:2006:Numerical-Optimization-Book},
given by 
\begin{align}
B_{k+1}= B_k - \frac{1}{\s_k^T B_k \s_k} B_k \s_k \s_k^T B_k + \frac{1}{\y_k^T \s_k} \y_k \y_k^T.
\label{eqn:bfgs}
\end{align}
The BFGS method generates positive-definite approximations whenever the initial approximation $B_0=\gamma_{k+1} I$ is positive definite and $\s_k^T \y_k > 0$.  A common value for $\gamma_{k+1}$ is $\y_k^T \y_k/\y_k^T \s_k$  \citep{Nocedal-Wright:2006:Numerical-Optimization-Book} (see \citep{Rafati-Marcia:2018:ICMLA} for alternative methods for choosing $\gamma_{k+1}$). 
Letting 
\begin{align}
S_k \triangleq [\s_0~~\dots~~\s_{k-1}] \text{ and } 
Y_k \triangleq[\y_{0}~~\dots~~\y_{k-1}],
\end{align}
the BFGS formula can be written in the following compact representation:
\begin{align}
B_{k} = B_0 + \Psi_k M_k \Psi_k^T,
\label{eq:compact}
\end{align}
where  $\Psi_k$ and $M_k$ are defined as
\begin{align}
\Psi_k = \Big[ B_0 S_k \ \ Y_k \Big],\quad
M_k = \begin{bmatrix} -S_k^T B_0 S_k& - L_k \\
-L_k^T & D_k 
\end{bmatrix}^{-1},
\label{eq:def:M-Psi}
\end{align}
and $L_k$ is the strictly lower triangular part and $D_k$ is the diagonal part of the matrix $S_k^T Y_k$, i.e., $S_k^T Y_k = L_k + D_k + U_k$, where $U_k$ is a strictly upper triangular matrix.    
(See \cite{ByrNS94} for further details.)


It is common in large-scale problems to store only the $m$ most-recently computed pairs $\{(\s_k, \y_k)\}$, where typically $m \le 100$. This approach is often referred to as \emph{limited-memory} BFGS (L-BFGS).

\subsection{Line-search L-BFGS Optimization}
In each iteration of the line-search (Algorithm \ref{Algo:line-search-pseudo-code}, we have to compute $p_k = - B_k^{-1} g_k$ at each iteration. We can make use of the following recursive formula for $H_k = B_k^{-1}$:
\begin{align}
H_{k+1} = \Big( I - \frac{\y_k \s_k^T}{\y_k^T \s_k} \Big)H_k \Big(I - \frac{\s_k \y_k^T}{\y_k \s_k^T} \Big) + \frac{\y_k \y_k^T}{\y_k \s_k^T},
\label{eq-L-BFGS-H} 
\end{align}
where $H_0 = \gamma_{k+1}^{-1} I = \nicefrac{\y_k^T \y_k}{\y_k^T \s_k} I$. The \emph{L-BFGS two-loop recursion algorithm}, given in Algorithm \ref{Algo:L-BFGS-two-loop-recursion}, can compute $p_k = -H_k g_k$ in $4mn$ operations \citep{Nocedal-Wright:2006:Numerical-Optimization-Book}. 

\begin{algorithm}
	\begin{algorithmic}
		\State $\q \gets g_k = \nabla \L(w_k)$
		\For{$i=k-1,\dots,k-m$}
		\State $\alpha_i = \frac{\s_i^T q}{\y_i^T \s_i}$
		\State $\q \gets \q - \alpha_i \y_i$
		\EndFor
		\State $\r \gets H_0 q$ 
		\For{$i=k-1,\dots,k-m$}
		\State $\beta = \frac{\y_i^T r}{\y_i^T \s_i}$
		\State $\r \gets \r + \s_i ( \alpha_i - \beta)$
		\EndFor\\
		\Return $- \r = - H_k g_k$
	\end{algorithmic}
	\caption{L-BFGS two-loop recursion.}
	\label{Algo:L-BFGS-two-loop-recursion}
\end{algorithm}

\subsection{Trust-Region Subproblem Solution}
To efficiently solve the trust-region subproblem \eqref{eq:TR-subproblem}, we exploit the compact representation of the BFGS matrix to obtain a global solution based on optimality conditions \eqref{eq:optcond}. In particular, we compute the spectral decomposition of $B_{k}$ using the compact representation of $B_{k}$. First, we obtain the $QR$ factorization of $\Psi_k=Q_kR_k$, where $Q_k$ has orthonormal columns and $R_k$ is strictly upper triangular.
Then we compute the eigendecomposition of $R_k M_k R_k^T = V_k \hat{\Lambda}_k V_k^T$, so that
\begin{align}
B_{k}=B_0+ \Psi_k M_k \Psi_k^T =\gamma_k I + Q_k V_k \hat{\Lambda}_k V_k^T Q_k^T.
\end{align}
Note that since $V_k$ is an orthogonal matrix, the matrix $Q_kV_k$ has orthonormal columns.
Let $P = [ \ Q_kV_k \ \ (Q_kV_k)^{\perp} \ ] \in \Re^{n \times n}$,
where $(Q_kV_k)^{\perp}$ is a matrix whose columns form an orthonormal
basis for the orthogonal complement of the range space of $Q_kV_k$, thereby making $P$ an orthonormal matrix.
Then 
\begin{align}
B_{k} = P 
\begin{bmatrix}
\hat{\Lambda} + \gamma_k I & 0 \\
0 & \gamma_k I
\end{bmatrix}
P^T.
\label{eq:orthonormal}
\end{align}	
Using this eigendecomposition to change variables
and diagonalize the first optimality condition in \eqref{eq:optcond},
a closed form expression for the solution $p_k^*$ can be derived.

The general solution for the trust-region subproblem using the Sherman-Morrison-Woodbury formula is given by
\begin{align}
p_k^* = - \frac{1}{\tau^*} 
\left [
I - \Psi_k( \tau^* M_k^{-1} + \Psi_k^T\Psi_k)^{-1}\Psi_k^T
\right ]
g_k,
\label{eq:SMW}
\end{align}  
where $\tau^* = \gamma_{k} + \sigma^{*}$, and $\sigma^*$ is the optimal Lagrange multiplier in \eqref{eq:optcond} (see \cite{brust2017solving} for details).

\section{Application to Image Recognition}
\label{sec:image-rec}
In this section, we compare the line search L-BFGS optimization method  with our proposed Trust-Region Minimization Algorithm for Training Responses (TRMinATR). The goal of the experiment is to perform the optimization necessary for neural network training. Both methods are implemented to train the LeNet-5 architecture with the purpose of image classification of the MNIST dataset.  All simulations were performed on an AWS EC2 p2.xlarge instance with 1 Tesla K80 GPU, 64 GiB memory, and 4 Intel 2.7 GHz Broadwell processors.  For the scalars $c_1$ and $c_2$ in the Wolfe line search condition, 
we used the typical values of $c_1 = 10^{-4}$ and $c_2 = 0.9$ \citep{Nocedal-Wright:2006:Numerical-Optimization-Book}.
All code is implemented in the Python language using TensorFlow and it is available at \texttt{https://rafati.net/lbfgs-tr}.

\subsection{LeNet-5 Convolutional Neural Network Architecture}
\label{sub:lenet}
We use the convolutional neural network architecture, LeNet-5 (Figure \ref{fig:lenet5}) for computing the likelihood $p_i(y_i|x_i;w_i)$. The LeNet-5 CNN is mainly used in the literature for character and digit recognition tasks \citep{LeCun-etal-1988:LeNet-5-original}. The details of layers' connectivity in LeNet-5 CNN architecture is given in Table \ref{t:LeNet}. The input to the network is $28 \times 28$ image and the output is $10$ neurons followed by a softmax function, attempting to approximate the posterior probability distribution $p(y_i | x_i;w)$. There are a total of $n = 431,080$ trainable parameters (weights) in LeNet-5 CCN.   
\begin{figure}[htb!]
	\centering
	\includegraphics[width=.9\textwidth]{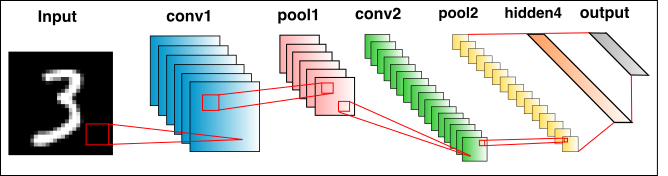}
	\caption{A LeNet deep learning network inspired by the architecture found in \cite{lecun2015lenet} . The neural network is used in the classification of the MNIST dataset of hand written digits. The convolutional neural network (CNN) uses convolutions followed by pooling layers for feature extraction. The final layer transforms the information into the required probability distribution.}\label{fig:net}
	\label{fig:lenet5}
\end{figure}

\begin{table}[htb!]
	\centering
	\caption{LeNet-5 CNN architecture \citep{LeCun-etal-1988:LeNet-5-original}.}
	\begin{tabular}[t]{@{}ll@{}} \hline 
		Layer &  Connections \\ \hline
		0: input & $28 \times 28$ image \\
		1 & \caja{t}{l}{convolutional, 20 $5 \times 5$ filters (stride $=1$), followed by ReLU} \\
		2 & \caja{t}{l}{max pooling, $2 \times 2$ window (stride $=2$)} \\
		3 & \caja{t}{l}{convolutional, 50 $5 \times 5$ filters (stride $=1$), followed by ReLU}\\
		4 & \caja{t}{l}{max pool, $2 \times 2$ window (stride $=2$) } \\
		5 & \caja{t}{l}{fully connected, 500 neurons (no dropout) \\ followed by ReLU} \\
		\caja{t}{l}{6: output} & \caja{t}{l}{fully connected, 10 neurons followed by softmax (no dropout)} \\ \hline
	\end{tabular}
	\label{t:LeNet}
\end{table}

\subsection{MNIST Image Classification Task}
\label{sub:mnist}
The convolutional neural network was trained and tested using the MNIST Dataset \citep{lecun1998mnist}. The dataset consists of 70,000 examples of handwritten digits with 60,000 examples used as a training set and 10,000 examples used as a test set. The digits range from 0 - 9 and their sizes have been normalized to 28x28 pixel images. The images include labels describing their intended classification. The MNIST dataset consists of $70,000$ examples of handwritten image of digits $0$ to $9$, with $N = 60,000$ image training set $\{(x_i,y_i)\}$, and $10,000$ used as the test set. Each image $x_i$ is a $28 \times 28$ pixel, and each pixel value is between $0$ and $255$. Each image $x_i$ in the training set include a label $y_i \in \{0,\dots,9\}$ describing its class. The objective function for the classification task in \eqref{eq:unconstrained-optimization} uses the cross entropy between model prediction and true labels given by
\begin{align}
\ell_i(w) = -\sum_{j=1}^{J} y_{ij} \log(p_i),
\end{align} 	
where the  $p_i(x_i;w) = p_i(y=y_i | x_i;w)$ is the probability distribution of the model, i.e., the likelihood that the image is correctly classified, $J$ is the number of classes ($J=10$ for MNIST digits dataset) and $y_{ij} = 1$ if $j=y_i$ and $y_{ij} = 0$ if $j \ne y_i$ (see \cite{Hastie:2009:Elements-of-statistical-learning-book} for details).

\subsection{Results}
The line search algorithm and TRMinATR perform comparably in terms of loss and accuracy. This remains consistent with different choices of the memory parameter $m$ (see Figure \ref{fig:lossacc}). The more interesting comparison is that of the training accuracy and the test accuracy. The two metrics follow each other closely. This is unlike the typical results using common gradient descent based optimization. Typically, the test accuracy is delayed in achieving the same results as the train accuracy. This would suggest that the model has a better chance of generalizing beyond the training data.

We compare the performance of L-BFGS method with the SGD one. The batch size of 64 was used for training. The loss and accuracy for training batch and the test data has reported for different learning rates, $10^{-6}\leq \alpha\leq 1.0$ (see Figure \ref{fig:lossacc} (c)--(f). Large learning rate (e.g. $\alpha=1$) led to poor performance (see Figure \ref{fig:lossacc} (c) and (d) ). The convergence was very slow for the small learning rates, i.e. $\alpha < 10^{-3}$. SGD only succeeds when the proper learning rate was used. This indicates that there is no trivial way for choosing the proper learning rate when using SGD methods and one can run so many experiments for fine-tuning the learning rates. Another interesting observation from Figure \ref{fig:lossacc} is that L-BFGS method (with either line-search or trust-region) has smaller \textit{generalization gap} in comparison to the SGD method. (Generalization gap is the difference between the expected loss and the empirical loss, or roughly the difference between test loss and train loss.)   

\begin{figure*}[htb!]
	\centering
	\begin{tabular}{cc} 
		\includegraphics[width=.44\textwidth]{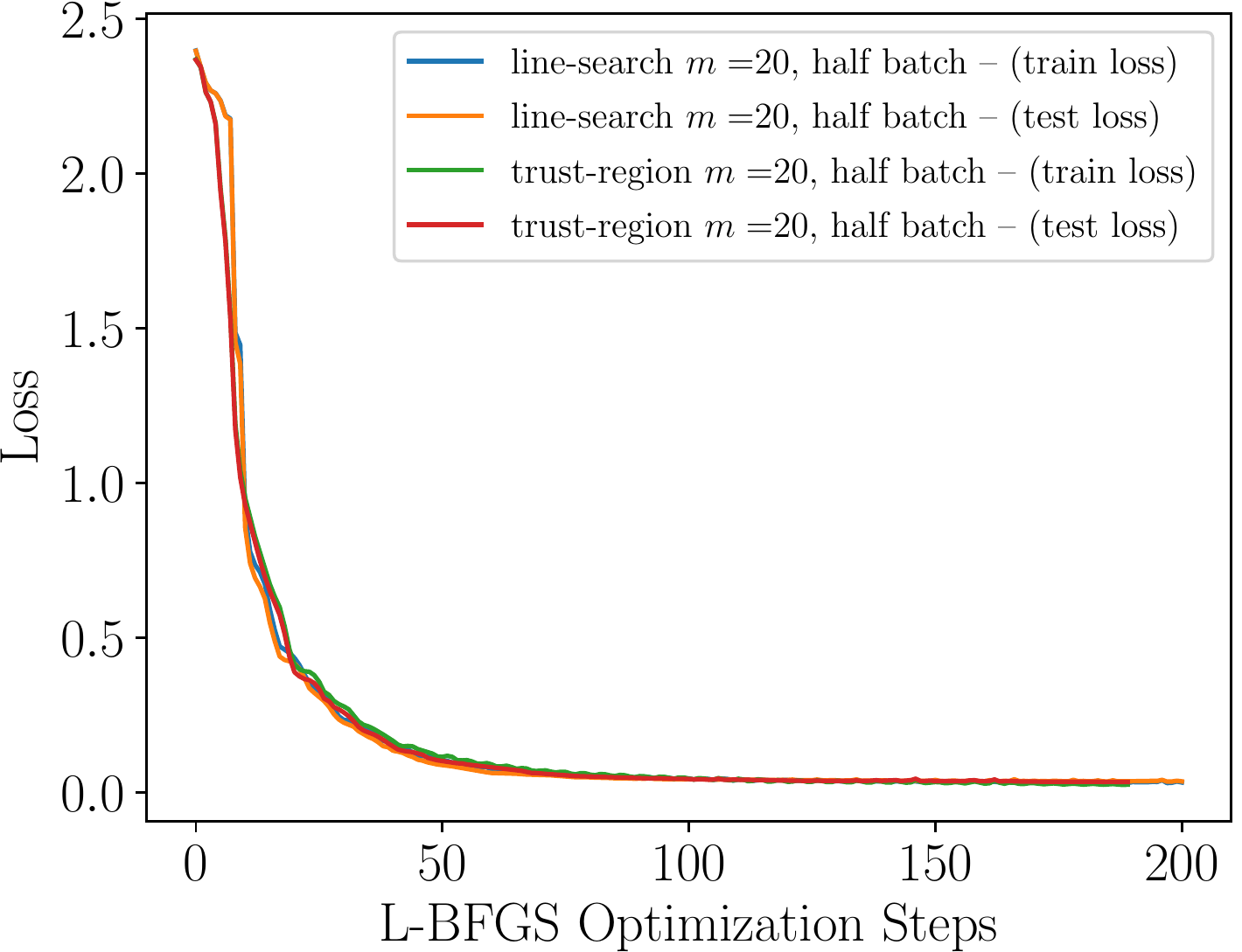} & 
		\includegraphics[width=.44\textwidth]{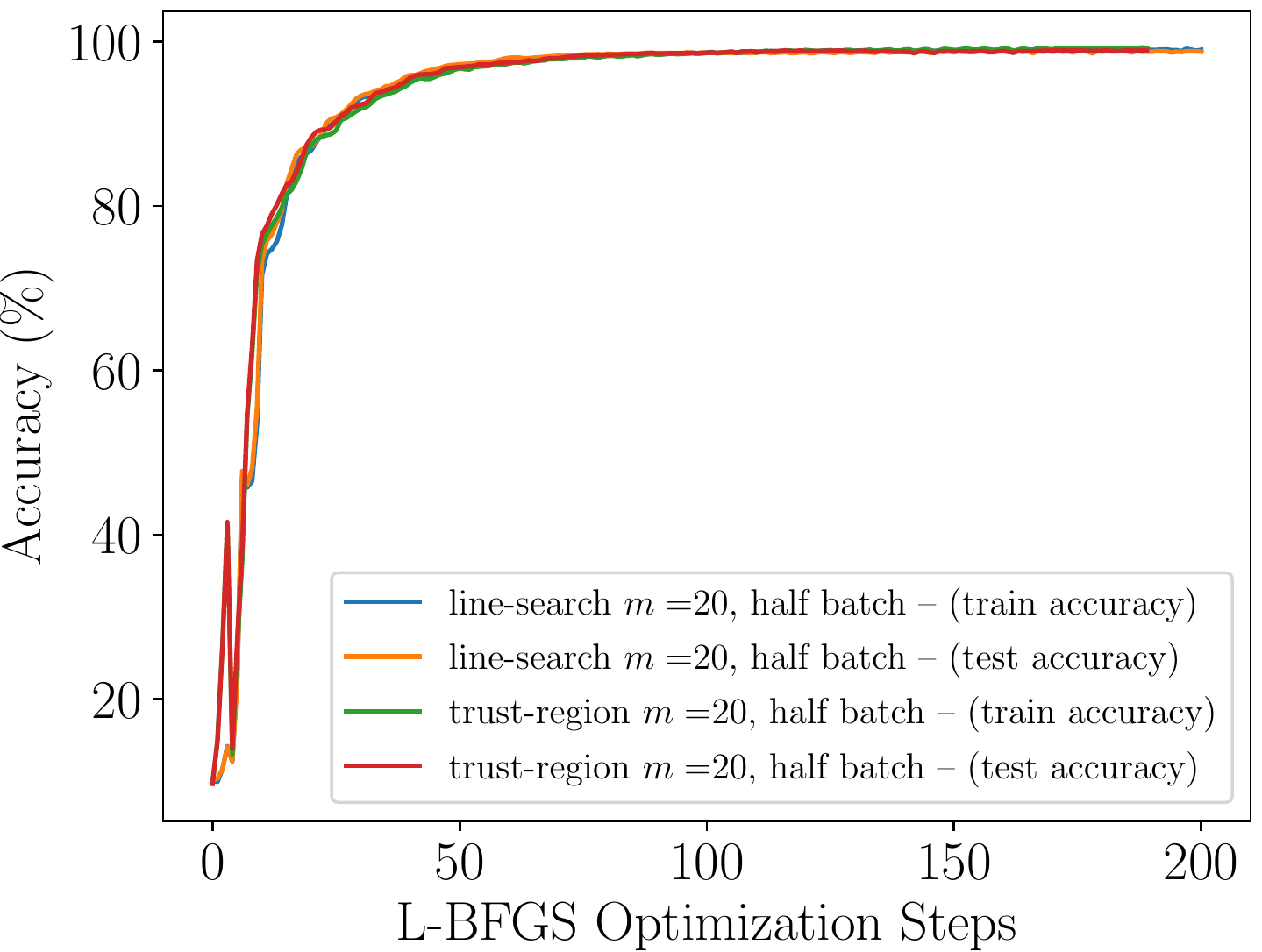}\\
		(a) & (b)\\
		\includegraphics[width=.44\textwidth]{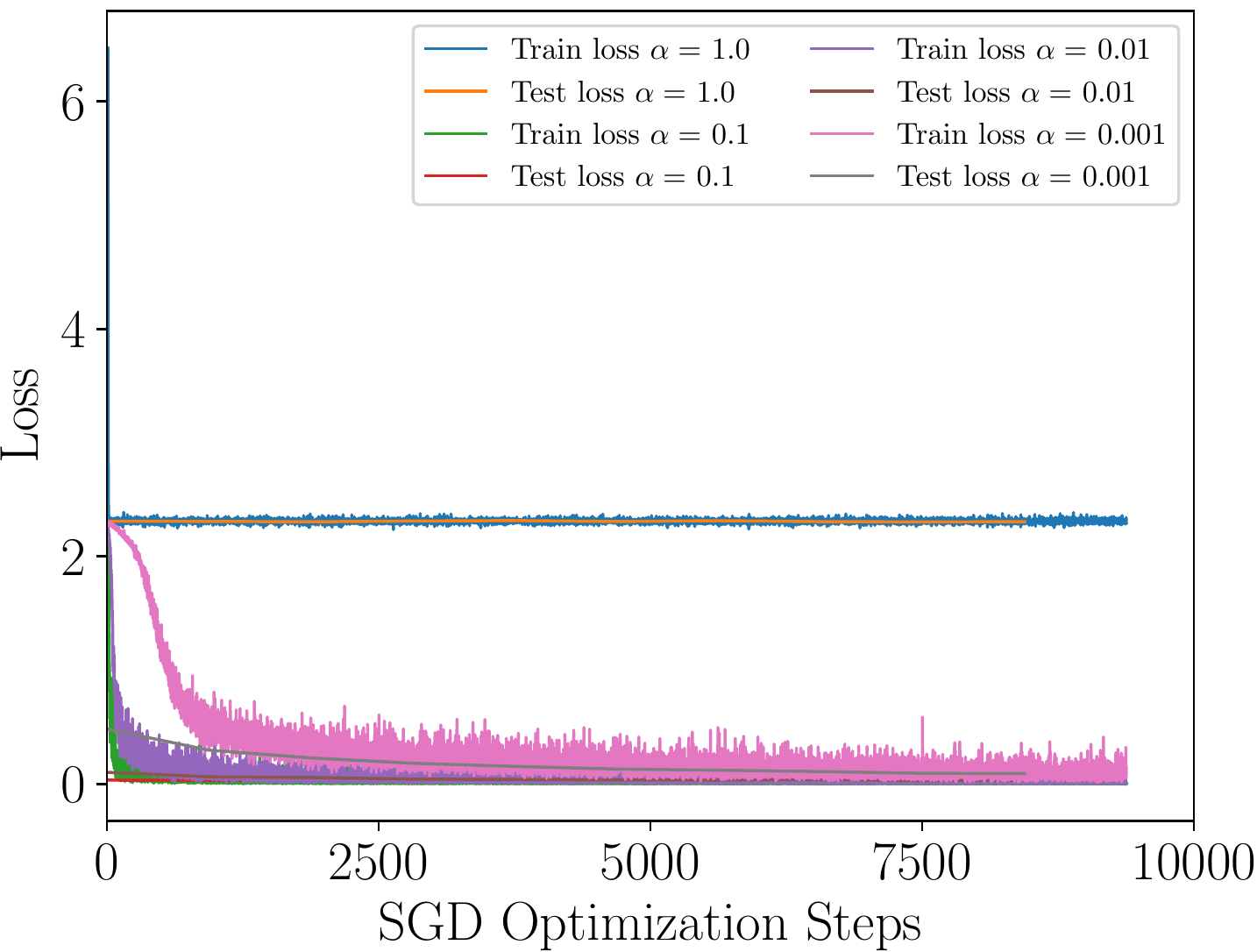} & 
		\includegraphics[width=.44\textwidth]{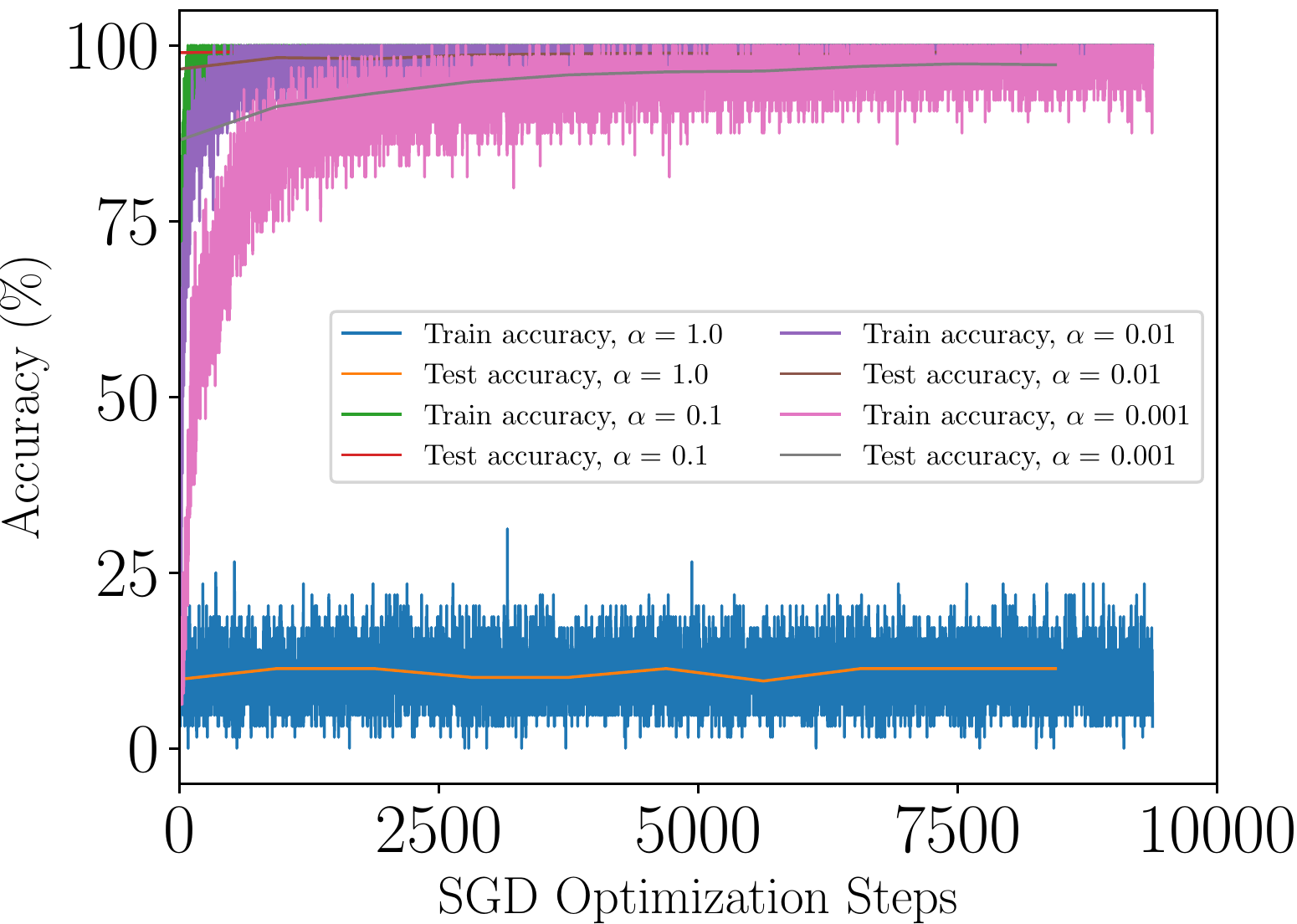}\\
		(c) & (d)\\
		\includegraphics[width=.44\textwidth]{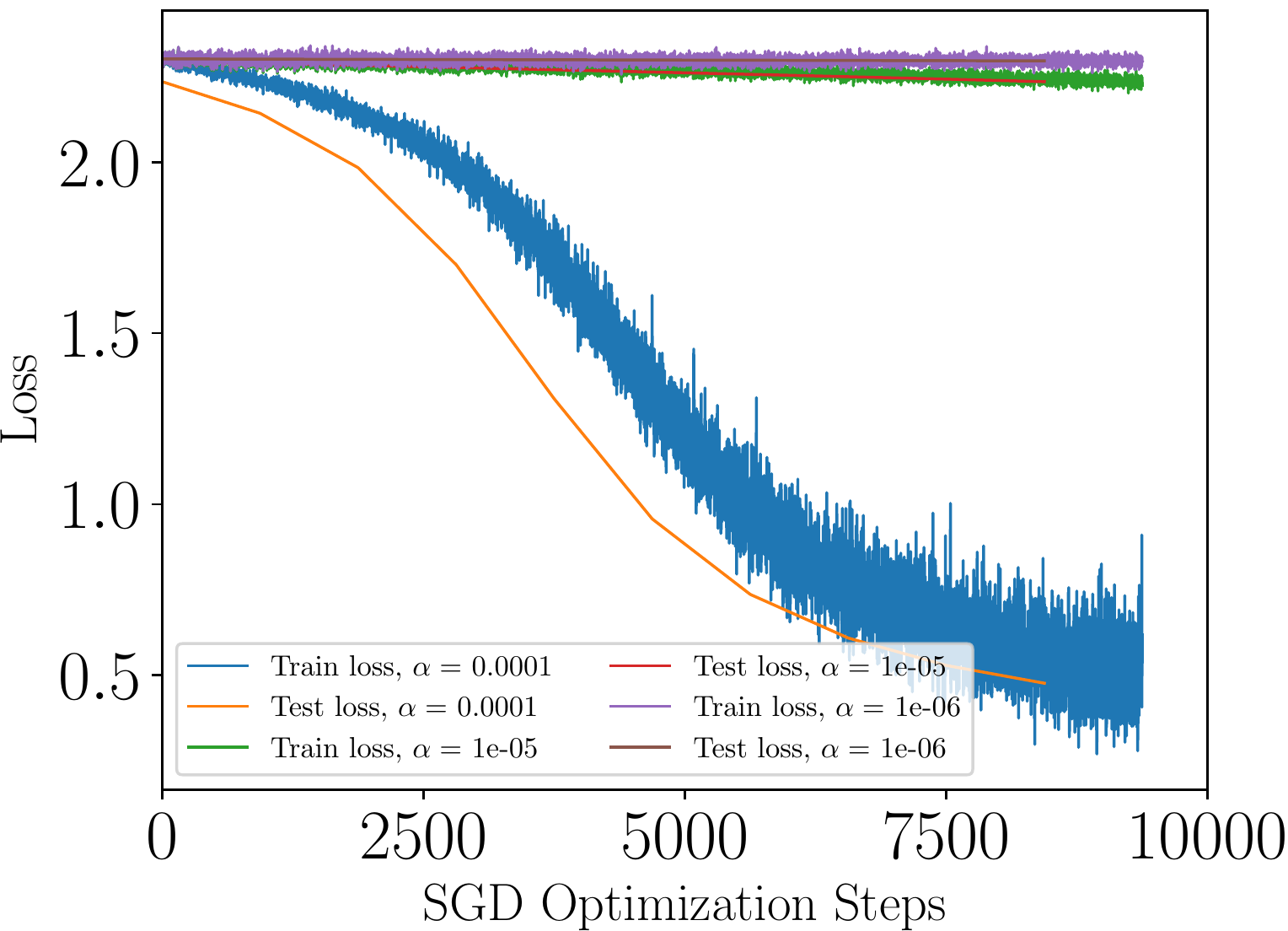} & 
		\includegraphics[width=.44\textwidth]{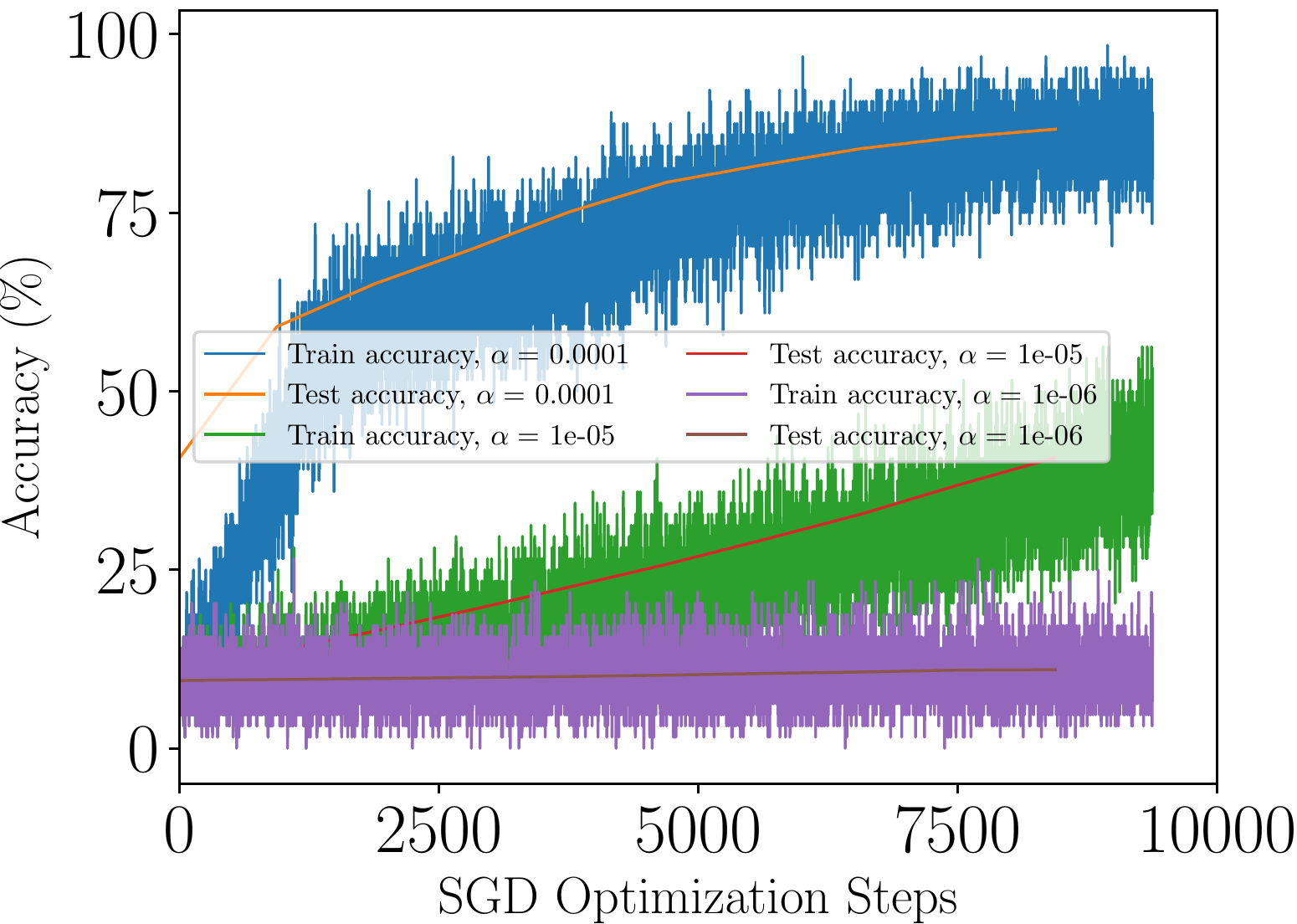}\\
		(e) & (f)
	\end{tabular}
	\caption{(a) \& (b) Loss and accuracy for the training and test sets, using L-BFGS line-search and L-BFGS trust-region methods for $m=20$ \citep{Rafati-et-al:2018:EUSIPCO}.   (c) \& (d) Loss and accuracy for the training and test sets using SGD with different learning rates $\alpha \in [1.0,0.1,0.01,0.001]$. (e) \& (f) Loss and accuracy for the training and test sets using SGD with small learning rates $\alpha \in [10^{-4},10^{-5},10^{-6}]$.}\label{fig:lossacc}
\end{figure*}

We also report that the TRMinATR significantly improves on the computational efficiency of the line-search method when using larger batch sizes. This could be the result of the line-search method's need to satisfy certain Wolfe conditions at each iteration. There is also an associated computational cost when verifying that the conditions for sufficient decrease are being met. When the batch size decreases, the trust-region method continues to outperform the line-search method. This is especially true when less information is used in the Hessian approximation (see Figure \ref{fig:time}).

\begin{figure}[htb!]
	\centering
	\includegraphics[width=0.6\textwidth]{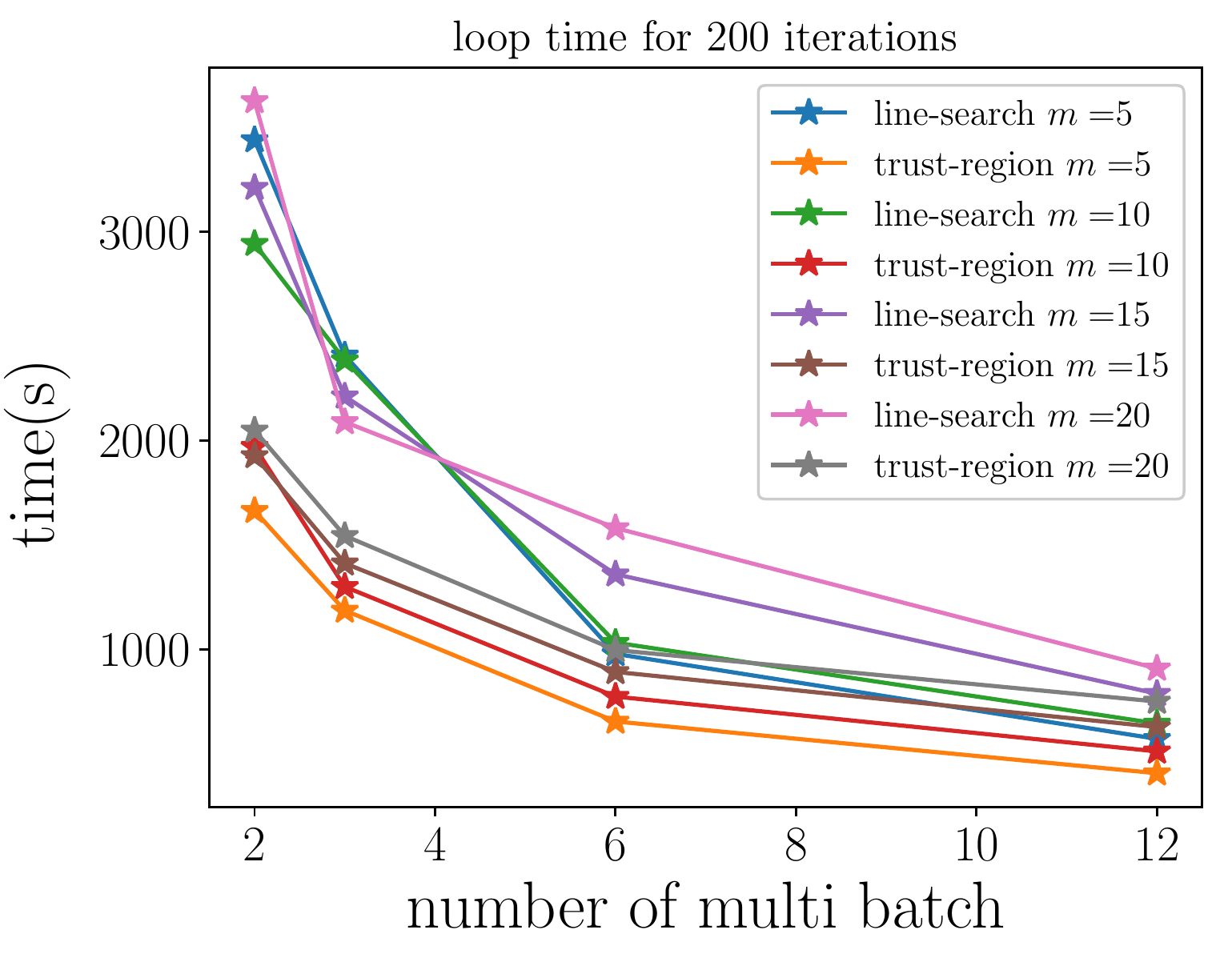}
	\caption{We compare the loop time for 200 iterations of the line-search and trust-region quasi-Newton algorithms for different batch sizes. As the number of multi batches increase, the size of each batch decreases. Both methods were tested using different values of the memory parameter $m$ \citep{Rafati-et-al:2018:EUSIPCO}. }\label{fig:time}
\end{figure}

\section{Application to Deep Reinforcement Learning}
\label{sec:rl}
\subsection{Reinforcement Learning Problem}
The reinforcement learning (RL) problem, -- a class of machine learning -- is that of learning through interaction with an \emph{environment}. The learner and decision maker is called the \emph{agent} and everything outside of the agent is called the \emph{environment}. The agent and environment interact over a sequence of discrete time steps, $t = 0,1,2,\dots,T$. At each time step, $t$, the agent receives a state, $s_t=s$, from the environment, takes an action, $a_t=a$, and one time step later, the environment sends a reward, $r_{t+1}=r \in \mathbb{R}$, and an updated state, $s_{t+1} = s'$ (see Figure \ref{f:agent-env}). Each cycle of interaction, $e = (s,a,r,s')$ is called a transition \emph{experience} (or a trajectory).
\begin{figure}[hbt!]
	\centering
	\includegraphics[width=0.5\textwidth]{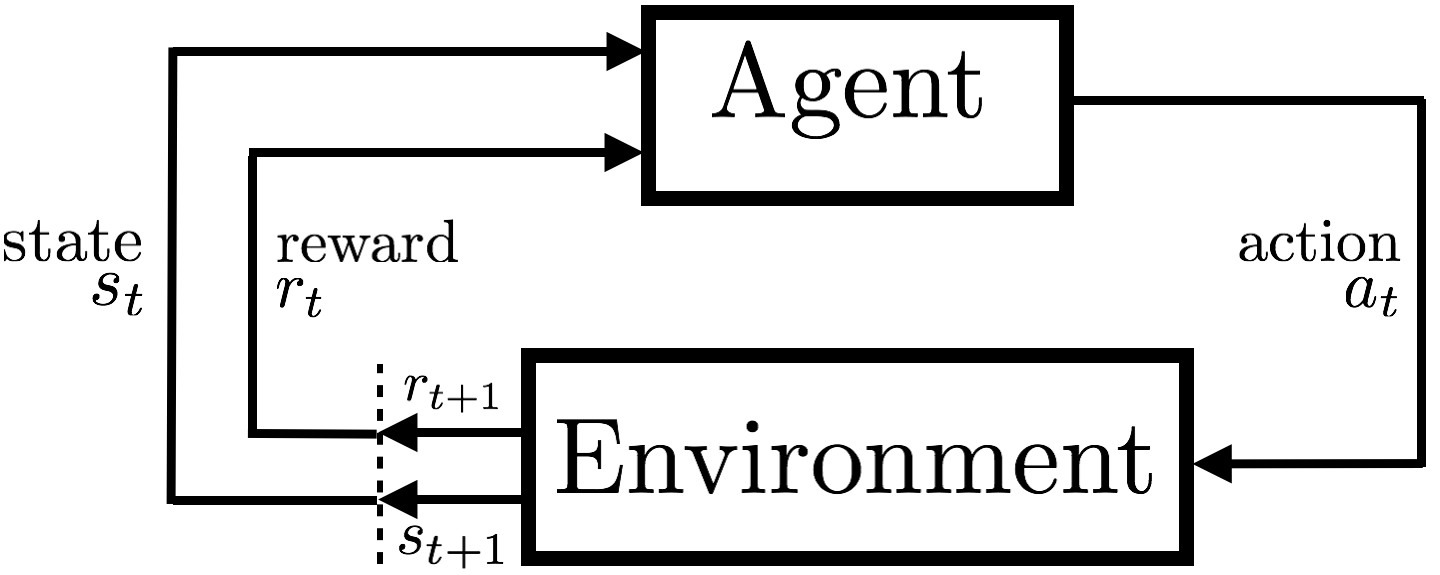}
	\caption{The agent/environment interaction in reinforcement learning. Adopted from \citep{RL-Book:Sutton:Barto:1998}}
	\label{f:agent-env}
\end{figure}
In an RL problem, the agent should implement a policy, $\pi$, from states, $\mathcal{S}$, to possible actions, $\mathcal{A}$. The objective of RL agent is to find an \emph{optimal policy} (best strategy), $\pi^*$, that maximizes its expected value of the \emph{return}, $G_t$, which is the cumulative sum of future rewards from the environment,
given by
\begin{align}
G_t \triangleq r_{t+1} + \gamma r_{t+2} + \gamma^2 r_{t+3} + \dots= \sum_{t'=t}^T \gamma^{t'-t} r_{t'+1},
\end{align}
where $\gamma \in (0,1]$ is a discount factor and $T \in \mathbb{N}$ is a final step (which can also be infinity)~\citep{RL-Book:Sutton:Barto:2017}. 
The optimal policy $\pi^*$ is defined as
\begin{align}
\pi^* = \arg\max\limits_{\pi} \mathbb{E}_{\pi} [ G_t ],
\label{eq:RL--optimization-probelm}
\end{align}

Reinforcement Learning is a class of solution methods for solving Markov Decision Processes (MDPs), when the agent does not have prior access to the environment models, i.e. the state transition probabilities, $\mathbb{P}(s'|s,a)$, and the reward function, $\mathcal{R}(s,a)$. Instead, the agent only perceives experiences (or trajectories) from interaction with the environment. The agent can save a limited memory of  past experiences (or history) in the set $\D$. It is important to note that each experience, $(s,a,s',r)$, is an example of the joint conditional probability distribution, $p(s',r|s,a)$. Thus the experience memory plays the role of the training data in RL.

It is often useful to define a parametrized value function $Q(s,a;w)$ to estimate the expected value of the return. Q-learning is a model-free RL algorithm that learns a policy without learning a model of the environment. Q-learning algorithm attempts to find the optimal value function by minimizing the expected risk, $L(w)$
\begin{align}
\min_{w \in \mathbb{R}^n} L(w) \triangleq \frac{1}{2}\mathbb{E}_{(s,a,r,s') \sim p} \Big[ \big( \Y - Q(s,a;w) \big)^2\Big],
\label{eq:expected-risk}
\end{align}
where $\Y = r + \max_{a'} Q(s',a';w)$ is the target value for the expected return based on the \emph{Bellman's optimality} equations \citep{RL-Book:Sutton:Barto:1998}. Once the value function is found, the policy function, $\pi$, can be found as
\begin{align}
\pi(s;w) = \argmax_{a \in \A}Q(s,a;w).
\end{align}

\subsection{Empirical Risk Minimization in Deep Reinforcement Learning}
In practice, the probability distribution over the trajectories, $p$, is unknown. Therefore, instead of minimizing the expected risk, $L(w)$ in \eqref{eq:expected-risk}, we can define an empirical risk minimization problem over a memory of agent's observed experiences, $\D$, as follows
\begin{align}
\min_{w \in \mathbb{R}^n} \L (w) \triangleq \frac{1}{2|\D|}\sum_{(s,a,r,s') \in \D} \Big[ \big( \Y - Q(s,a;w) \big)^2\Big].
\label{eq:empirical-risk}
\end{align}

At each optimization step, $k$, a small set of experiences, $J_k$, are randomly sampled from the \emph{experience replay memory}, $\D$. This sample is used to compute an stochastic gradient of the objective function, $\nabla \L (w)^{J_k}$, as an approximation for the true gradient, $\nabla \L (w)$, 
\begin{align}
\nabla \L (w)^{(J_k)} \triangleq \frac{-1}{|J_k|}\sum_{e \in J_k} \Big[ \big( \Y - Q(s,a;w) \big) \nabla Q \Big].
\label{eq:stochastic-gradient}
\end{align}

\subsection{L-BFGS line-search deep Q-learning method}
In this section, we propose a novel algorithm for the optimization problem in deep Q-Learning framework, based on the limited-memory BFGS method using a line search strategy. This algorithm is designed to be efficient for parallel computations on GPU. Also the experience memory $\D$ is emptied after each gradient computation, hence the algorithm needs much less RAM memory. 

Inspired by \cite{Berahas2016multibatch}, we use the overlap between the consecutive multi-batch samples $O_k = J_k \cap J_{k+1}$ to compute $y_k$  as 
\begin{align}
\y_k = \nabla \L(w_{k+1})^{(O_k)} -  \nabla \L(w_{k})^{(O_k)}.
\label{eq:y-overlap}
\end{align} 
The use of overlap to compute $y_k$ has been shown to result in more robust convergence in L-BFGS since L-BFGS uses  gradient differences to update the Hessian approximations (see \cite{Berahas2016multibatch} and \cite{Erway-etal-2018-arXiv}).

At each iteration of optimization we collect experiences in $\D$ up to batch size $b$ and use the entire experience memory $\D$ as the overlap of consecutive samples $O_k$. 
For computing the gradient $g_k = \nabla \L(w_k)$, we use the $k$th sample, $J_k = O_{k-1} \cup O_k$
\begin{align}
\nabla \L (w_k)^{(J_k)} = \frac{1}{2}(\nabla \L (w_k)^{(O_{k-1})} + \nabla \L (w_k)^{(O_{k})}).
\label{eq:gradient-Jk}
\end{align}    
Since $\nabla \L(w_k)^{(O_{k-1})}$ is already computed to obtain $\y_{k-1}$ in the previous iteration, we only need to compute $\nabla \L^{(O_k)}(w_{k})$, given by 
\begin{align}
\nabla \L (w_k)^{(O_k)} = \frac{-1}{|\D|}\sum_{e \in D} \Big[ \big( \Y - Q(s,a;w_k) \big) \nabla Q \Big].
\label{eq:overlap-gradient}
\end{align}    
Note that in order to obtain $\y_k$, we only need to compute $\nabla \L (w_{k+1})^{(O_{k})}$ since $\nabla \L (w_k)^{(O_k)}$ is already computed when we computed the gradient in \eqref{eq:gradient-Jk}. 

The line search multi-batch L-BFGS optimization algorithm for deep Q-Leaning is provided in Algorithm \ref{Algo:DQN+L-BFGS}.
\begin{algorithm}
	\begin{algorithmic}
		\State \textbf{Inputs:} batch size $b$, L-BFGS memory $m$, exploration rate $\epsilon$ 
		\State \textbf{Initialize} experience memory $\mathcal{D} \gets \emptyset$ with capacity $b$
		\State \textbf{Initialize} $w_0$, i.e. parameters of $Q(.,.;w)$ randomly
		\State \textbf{Initialize} optimization iteration $k \gets 0$
		\For{ episode $=1,\dots,M$} 
		\State Initialize state $s \in \mathcal{S}$
		\Repeat{ for each step $t = 1,\dots,T$} 					
		\State compute $Q(s,a;w_k)$
		\State $a\gets$\texttt{EPS-GREEDY}$(Q(s,a;w_k),\epsilon)$
		\State Take action $a$
		\State Observe next state $s'$ and external reward $r$ 
		\State Store transition experience $e=\{s,a,r,s'\}$ to $\mathcal{D}$
		\State $s \gets s'$
		\Until{$s$ is terminal or intrinsic task is done}
		\If{$|\mathcal{D}| == b$}
		\State $O_k \gets \mathcal{D}$
		\State Update $w_k$ by performing \textbf{optimization step}
		\State $\mathcal{D} \gets \emptyset$
		\EndIf
		\EndFor			
	\end{algorithmic}
	\textbf{========================================}\\
	\textbf{Multi-batch line search L-BFGS Optimization step:}
	\begin{algorithmic}
		\State Compute gradient $g^{(O_k)}_k$
		\State Compute gradient $g^{(J_k)}_k \gets \frac{1}{2} g^{(O_k)}_k + \frac{1}{2} g^{(O_{k-1})}_k$
		\State Compute $p_k = - B_k^{-1} g^{(J_k)}_k$ using Algorithm \ref{Algo:L-BFGS-two-loop-recursion} 
		\State Compute $\alpha_k$ by satisfying the Wolfe Conditions \eqref{eqn:Wolfe-Conditions} 
		\State Update iterate $w_{k+1} = w_{k} + \alpha_k p_{k}$
		\State $\s_{k} \gets w_{k+1} - w_{k}$
		\State Compute $g^{(O_k)}_{k+1} = \nabla \L(w_{k+1})^{(O_k)}$ 
		\State $\y_{k} \gets g^{(O_k)}_{k+1} - g^{(O_k)}_{k}$
		\State Store $\s_{k}$ to $S_{k}$ and $\y_{k}$ to $Y_k$ and remove  oldest pairs
		\State $k \gets k+1$
		
	\end{algorithmic}
	\caption{Line search Multi-batch L-BFGS Optimization for Deep Q Learning.}
	\label{Algo:DQN+L-BFGS}
\end{algorithm}

\subsection{Convergence Analysis}
\label{sec:convergence-analysis}
In this section, we present a convergence analysis for our deep Q-learning with multi-batch line-search L-BFGS optimization method (Algorithm \ref{Algo:DQN+L-BFGS}). We also provide an analysis for optimality of the state action value function. We then provide a comparison between the computation time of our deep L-BFGS Q-learning method (Algorithm \ref{Algo:DQN+L-BFGS}) and that of DeepMind's Deep Q-learning algorithm \citep{DeepMind:Nature:2015},  which uses a variant of the SGD method.   

\subsection{Convergence for the Empirical Risk}
To analyze the convergence properties of empirical risk function $\L(w)$ in \eqref{eq:empirical-risk} we assume that
\begin{subequations}
	\begin{align}
	&\L(w) \textrm{ is strongly convex and twice differentiable}. \label{eq:assumption-1} \\
	&\textrm{Given } w,\textrm{ there are } \lambda,\Lambda>0 \textrm{ s.t. } \lambda I \preceq \nabla ^2\L(w) \preceq \Lambda I, \label{eq:assumption-2}\\
	&\textrm{Given } w,\textrm{ there is } \eta > 0 \textrm{ such that } \| \nabla \L(w) \|^2 \leq \eta^2. \label{eq:assumption-3}
	\end{align}
	\label{eq:assumptions}
\end{subequations}
In \eqref{eq:assumption-2} we assume that the eigenvalues of the Hessian matrix are bounded, and in \eqref{eq:assumption-3} we assume that gradients are not unbounded. 
\begin{lemma}
	Given $w$, there exist $\lambda',\Lambda'>0$ such that $\lambda' I \preceq H_k \preceq \Lambda' I$.
	\label{lemma:1}
\end{lemma}

\begin{proof}
	Due to the assumptions \eqref{eq:assumption-1} and \eqref{eq:assumption-2}, the eigenvalues of the positive-definite matrix $H_k$ are also bounded \citep{Berahas2016multibatch,byrd2016stochastic}.   
\end{proof}

\begin{lemma}
	Let $w^*$ be a minimizer of $\L$. Then, for all $w$, we have $2 \lambda (\L(w) - \L(w^*) \leq \| \nabla \L(w) \|^2$. 
	\label{lemma:2}
\end{lemma}

\begin{proof}
	For any convex function, $\L$ and for any two points, $w$ and $w^*$, one can show that
	\begin{align}
	\L(w) \leq \L(w^*) + \nabla \L(w^*)^T (w-w^*) + \frac{1}{2\lambda}\| \nabla \L(w) - \nabla \L(w^*)\|^2.
	\label{eq:nestrov}
	\end{align}
	(see  \citep{Nesterov:2013}).  
	Since $w^*$ is a minimizer of $\L$,  $\nabla \L(w^*)=0$ in \eqref{eq:nestrov}, which completes the proof.
\end{proof}

\begin{theorem}
	Let $w_k$ be iterates generated by Algorithm \ref{Algo:DQN+L-BFGS}, and assume that the step length, $\alpha_k$, is fixed. The upper bound for the  empirical risk offset from the true minimum value is 
	\begin{align}
	\begin{split}
	\L(w_k) - \L(w^*) \leq (1 - 2 \alpha \lambda \lambda' )^k \big[\L(w_0) - \L(w^*)\big] \\+ \big[1 - (1 - 2 \alpha \lambda \lambda')^k\big]\frac{\alpha^2 \Lambda'^2 \Lambda \eta^2}{4 \lambda' \lambda}.
	\end{split}
	\label{eq:L-bound}
	\end{align}  
\end{theorem}

\begin{proof}
	Using the Taylor expansion of
	\[
	\L(w_{k+1}) = \L(w_k - \alpha_k H \nabla \L(w_k)) 
	\]
	around $w_k$, we have
	\begin{align}
	\L(w_{k+1}) \leq \L(w_{k}) - \alpha_k \nabla \L(w_k)^T H_k \nabla \L(w_k) + \frac{\Lambda}{2} \|  \alpha_k \nabla \L(w_k)^T H_k \nabla \L(w_k) \|^2.
	\end{align}
	By applying assumptions \eqref{eq:assumptions} and Lemmas \ref{lemma:1} and \ref{lemma:2} to the above inequality, we have
	\begin{align}
	&\L(w_{k+1}) \leq \L(w_{k}) - 2 \alpha_k \lambda' \lambda [\L(w_k) - \L(w^*)] + \frac{\alpha_k^2 \Lambda'^2 \Lambda \eta^2}{4 \lambda' \lambda}
	\end{align}
	By rearranging terms and using  recursion expression and recursion over $k$ we have the proof. For a more detailed proof see \cite{byrd2016stochastic} and \cite{Berahas2016multibatch}.    
\end{proof}
If the step size is bounded, $\alpha\in(0,1/2\lambda\lambda')$, we can conclude that the first term of the bound given in \eqref{eq:L-bound} is decaying linearly to zero when $k \to \infty$ and the constant residual term, $\frac{\alpha^2 \Lambda'^2 \Lambda \eta^2}{4 \lambda' \lambda}$, is the neighborhood of convergence.

\subsection{Value Optimality}
The Q-learning method has been proved to converge to the optimal value function if the step sizes satisfies $\sum_{k}\alpha_k = \infty$ and $\sum_{k}\alpha_k^2 < \infty$ \citep{Jaakola:1994:q-learning-convergence}. 
Now, we want to prove that Q-learning using the L-BFGS update also theoretically converges to the optimal value function under one additional condition on the step length, $\alpha_k$.
\begin{theorem}
	Let $Q^*$ be the optimal state-action value function and $Q_{k}$ be the Q-function with parameters $w_{k}$. Furthermore, assume that the gradient of $Q$ is bounded, $\| \nabla Q \|^2 \leq \eta''^2$, and the Hessian of $Q$ functions satisfy $\lambda''\preceq \nabla^2 Q \preceq \Lambda''$. We have
	\begin{align}
	\| Q_{k+1} -  Q^* \|_{\infty} <  \prod_{j=0}^k \Big[1 - \alpha_j \eta''^2 \lambda + \frac{\alpha_j \eta''^2 \Lambda'^2 \Lambda''}{2}\Big]^{k} \| Q_{0} -  Q^* \|_{\infty}.
	\label{eq:value-bound}
	\end{align}
	If step size $\alpha_k$ satisfies
	\begin{align}
	\Big|1 - \alpha_k \eta''^2 \lambda + \frac{\alpha_k \eta\eta' \Lambda'^2 \Lambda''}{2}\Big| \le \mu <1, \quad \text{for all $k$},
	\label{eq:alpha-cond-value-bound}
	\end{align}
	$Q(.,.;w_k)$ ultimately will converge to $Q^*$, as $k \to \infty$. 
	\label{theorem:value-optimality}
\end{theorem}  

\begin{proof}
	First we derive the effect of the parameter update from $w_k$ to 
	\[
	w_{k+1} = w_k - \alpha_k H_k \nabla \L (w_k)
	\]
	on the optimality neighbor.       
	\begin{align}
	\| Q_{k+1} -  Q^* \|_{\infty} \triangleq \max_{s,a} \big| Q(s,a,w_{k+1}) - Q^*(s,a) \big|
	\end{align}
	We approximate the gradient using only one experience $(s,a,r,s')$,
	\begin{align}
	\nabla \L(w_k) \approx \big(Q(s,a;w_k) - Q^*(s,a;w_k) \big)\nabla Q_k(s,a;w_k),
	\label{eq:one-ex-grad}
	\end{align}
	Using Taylor's expansion to approximate $Q(s,a,w_{k+1})$ results in
	\begin{align}
	\begin{split}
	&Q(s,a;w_{k+1}) = Q(s,a;w_{k} - \alpha_k H_k \nabla \L(w_k)) \\
	&= Q(s,a;w_k) -  \alpha_k  \nabla \L_k^T H_k \nabla Q_k + \frac{\alpha_k^2}{2}  \nabla \L_k^T H_k \nabla^2 Q(\xi_k) H_k \nabla \L_k^T \\
	&= Q_k -  \alpha_k  (Q_k-Q^*) \nabla Q_k^T H_k \nabla Q_k + \frac{\alpha_k^2}{2} (Q_k-Q^*) \nabla Q_k^T H_k \nabla^2 Q(\xi_k) H_k \nabla \L_k^T,
	\end{split}
	\end{align}
	where $\xi$ is between $w_k$ and $w_{k+1}$,
	$Q_k \coloneqq Q(s,a;w_k)$, $\nabla Q_k \coloneqq \nabla Q(s,a;w_k)$, and $\nabla \L_k \coloneqq \nabla \L(w_k)$. We can use the above expression to compute $\| Q_{k+1} -  Q^* \|_{\infty}$:
	\begin{align}
	\begin{split}
	&\| Q_{k+1} -  Q^* \|_{\infty}= \\ 
	&\max_{s,a} \Big| (Q_k-Q^*)  {\Big[1 -\alpha_k \nabla Q_k^T H_k \nabla Q_k + \frac{\alpha_k^2}{2} \nabla Q_k^T H_k \nabla^2 Q(\xi_k) H_k \nabla \L_k\Big] \Big|_{(s,a)}}.
	\end{split}
	\label{eq:q-optimal-2}
	\end{align}
	If $\alpha_k$ satisfies
	\begin{align}
	\Big|1 -\alpha_k \nabla Q_k^T H_k \nabla Q_k +\frac{\alpha_k^2}{2} \nabla Q_k^T H_k \nabla^2 Q_k H_k \nabla \L_k\Big| 
	\le \mu < 1,
	\label{eq:cond-alpha-conv}
	\end{align}
	then 
	\begin{align}
	\| Q_{k+1} -  Q^* \|_{\infty} \le \mu \| Q_{k} -  Q^* \|_{\infty} \le \mu^{k+1}  \| Q_{0} -  Q^* \|_{\infty}.
	\label{eq:q-optimal-3}
	\end{align} 
	Therefore, $Q_k$ converges to $Q^*$ when $k \to \infty$. Considering our assumptions on the bounds of the eigenvalues of $\nabla^2 Q_k$ and $H_k$, we can derive \eqref{eq:alpha-cond-value-bound} from \eqref{eq:cond-alpha-conv}. Recursion on \eqref{eq:q-optimal-2} from $k=0$ to $k+1$ results in \eqref{eq:value-bound}.  
\end{proof}     
\subsection{Computation Time}
Let us compare the cost of deep L-BFGS Q-learning in Algorithm \ref{Algo:DQN+L-BFGS} with DQN algorithm in \citep{DeepMind:Nature:2015} that uses a variant of SGD. Assume that the cost of computing gradient is $\O(bn)$ where $b$ is the batch size. The real cost is probably less than this due to the parallel computation on GPUs. Let's assume that we run both algorithm for $L$ steps. We update the weights every $b$ steps. Hence there is $L/b$ maximum updates in our algorithm. The SGD batch size in \cite{DeepMind:Nature:2015}, $b_s$, is smaller than $b$, but the frequency of the update is high, $f \ll b$. Each iteration of the L-BFGS algorithm update introduces the cost of computing the gradient, $g_k^{(O_k)}$, which is $\mathcal{O}(bn)$, the cost of computing the search step, $p_k = -H_k g_k^{(O_k)}$, using L-BFGS two-loop recursion (Algorithm \ref{Algo:L-BFGS-two-loop-recursion}), which is $\mathcal{O}(4mn)$, and the cost of satisfying the Wolfe conditions \eqref{eqn:Wolfe-Conditions} to find a step size that usually satisfies for $\alpha=1$ and, in some steps, requires recomputing the gradient $z$ times. Therefore we have
\begin{align}
\begin{split}
\frac{\textrm{Cost of Algorithm \ref{Algo:L-BFGS-two-loop-recursion}}}{\textrm{Cost of DQN \citep{DeepMind:Nature:2015}}} = \frac{(L/b)(zbn + 4mn)}{(L/f)(b_s n)}
= \frac{fz}{b_s} + \frac{4fm}{b b_s}.
\end{split}
\end{align}
In our algorithm, we use a quite large batch size to compute less noisy gradients. With $b = 2048$, $b_s = 32$, $f = 4$, $z=5$, $m=20$, the runtime cost ratio will be around $0.63 < 1$. Although the per-iteration cost of the the SGD algorithm is lower than L-BFGS, the total training time of our algorithm is less than DQN \citep{DeepMind:Nature:2015} for the same number of RL steps due to the need for less frequent updates in the L-BFGS method.                

\subsection{Experiments on ATARI 2600 Games}
\label{sec:experiment}
We performed experiments using our approach (Algorithm \ref{Algo:DQN+L-BFGS}) on six ATARI 2600 games -- Beam-Rider, Breakout, Enduro, Q*bert, Seaquest, and Space Invaders. We used OpenAI's gym ATARI environments \citep{OpenAI} which are wrappers on the Arcade Learning Environment emulator \citep{Bellemare:2013:ALE}. These games have been used by other researchers investigating different learning methods 
\citep{DeepMind:Nature:2015,Bellemare:2013:ALE,Bellemare:2012:Contingency,Hausknecht:2014:HNeat-ATARI,Schulman:2015:TRPO-ATARI}, and, hence, they serve as benchmark environments for the evaluation of deep reinforcement learning algorithms. 

We used DeepMind's Deep Q-Network (DQN) architecture, described in \cite{DeepMind:Nature:2015}, as a function approximator for $Q(s,a;w)$. The same architecture was used to train agents to play the different ATARI games. The raw Atari frames, which are $210 \times 160$ pixel images with a $128$ color palette, were preprocessed by first converting their RGB representation to gray-scale and then down-sampling the images to be $110\times84$ pixels. The final input representation is obtained by cropping an $84 \times 84$ region of the image that roughly captures the playing area. The stack of the last 4 consecutive frames was used to produce the input, of size  $(4 \times 84 \times 84)$, to the $Q$-function. The first hidden layer of the network consisted of $32$ convolutional filters of size $8 \times 8$ with stride $4$, followed by a Rectified Linear Unit (ReLU) for nonlinearity. The second hidden layer consisted of $64$ convolutional filters of size $4 \times 4$ with stride 2, followed by a ReLU function. The third layer consisted of 512 fully-connected linear units, followed by ReLU. The output layer was a fully-connected linear layer with an output, $Q(s,a_i,w)$, for each valid joystick action, $a_i \in \mathcal{A}$. The number of valid joysticks actions, i.e. $|\mathcal{A}|$, was 9 for Beam-Rider, 4 for Breakout, 9 for Enduro, 6 for Q*Bert, 18 for Seaquest, and 6 for Space-Invaders.  

We only used 2 million ($2000\times 1024$) Q-learning training steps for training the network on each game (instead of $50$ million steps that was used originally in \cite{DeepMind:Nature:2015}). The training was stopped when the norm of the gradient, $\| g_k \|$, was less than a threshold. We used $\epsilon$-greedy for an exploration strategy, and, similar to  \cite{DeepMind:Nature:2015}, the exploration rate, $\epsilon$, was annealed linearly from $1$ to $0.1$. 

Every 10,000 steps, the performance of the learning algorithm was tested by freezing the Q-network's parameters. During the test time, we used $\epsilon=0.05$. The greedy action, $\max_a Q(s,a;w)$, was chosen by the Q-network $95\%$ of the times and there was $5\%$ randomness, similar to the DeepMind implementation in \cite{DeepMind:Nature:2015}.      

Inspired by \cite{DeepMind:Nature:2015}, we also used separate networks to compute the target values, $\Y = r + \gamma \max_{a'} Q(s',a',w_{k-1})$, which was essentially the network with parameters in previous iterate. After each iteration of the multi-batch line search L-BFGS, $w_k$ was updated to $w_{k+1}$, and the target network's parameter $w_{k-1}$ was updated to $w_k$.

Our optimization method was different than DeepMind's RMSProp method, used in \cite{DeepMind:Nature:2015} (which is a variant of SGD). We used a stochastic line search L-BFGS method as the optimization method (Algorithm \ref{Algo:DQN+L-BFGS}). There are a few important differences between our implementation of deep reinforcement learning and DeepMind's DQN algorithm. 

We used a quite large batch size, $b$, in comparison to \cite{DeepMind:Nature:2015}. We experimented with different batch sizes $b\in$ \{512, 1024, 2048, 4096, 8192\}. The experience memory, $\D$, had a capacity of $b$ also. We used one NVIDIA Tesla K40 GPU with 12GB GDDR5 RAM. The entire experience memory, $\D$, could fit in the GPU RAM with a batch size of $b \leq 8192$. 

After every $b$ steps of interaction with the environment, the optimization step in Algorithm \ref{Algo:DQN+L-BFGS} was ran. We used the entire experience memory, $\D$, for the overlap, $O_k$, between two consecutive samples, $J_k$ and $J_{k+1}$, to compute the gradient in \eqref{eq:overlap-gradient} as well as $\y_k$ in \eqref{eq:y-overlap}. Although the DQN algorithm used a smaller batch size of $32$, the frequency of optimization steps was high (every $4$ steps). We hypothesize that using the smaller batch size made the computation of the gradient too noisy, and, also, this approach doesn't save significant computational time, since the overhead of data transfer between GPU and CPU is more costly than the computation of the gradient over a bigger batch size, due to the power of parallelism in a GPU. Once the overlap gradient, $g_k^{(O_k)}$, was computed, we computed the gradient, $g_k^{(J_k)}$, for the current sample, $J_k$, in \eqref{eq:gradient-Jk} by memorizing and using the gradient information from the previous optimization step. Then, the L-BFGS two loop-recursion in Algorithm \ref{Algo:L-BFGS-two-loop-recursion} was used to compute the search direction $p_k = - H_k g_k^{(J_k)}$.

After finding the quasi-Newton decent direction, $p_k$, the Wolfe Condition \eqref{eqn:Wolfe-Conditions} was applied to compute the step size, $\alpha_k \in [0.1,1]$, by satisfying the sufficient decrease and the curvature conditions \citep{Wolfe1969,Nocedal-Wright:2006:Numerical-Optimization-Book}. During the optimization steps, either the step size of $\alpha_k=1$ satisfied the Wolfe conditions, or the line search algorithm iteratively used smaller $\alpha_k$ until it satisfied the Wolfe conditions or reached a lower bound of $0.1$. The original DQN algorithm used a small fixed learning rate of $0.00025$ to avoid the execrable drawback of the noisy stochastic gradient decent step, which makes the learning process very slow.  

The vectors $\s_k = w_{k+1} - w_k$ and $\y_k = g^{(O_k)}_{k+1} - g^{(O_k)}_{k}$ were only added to the recent collections $S_k$ and $Y_k$  if $\s_k^T \y_k > 0$ and not close to zero. We applied this condition  to \emph{cautiously} preserve the positive definiteness of the L-BFGS matrices $B_k$. Only the $m$ recent $\{(\s_i,\y_i)\}$ pairs were stored into $S_k$ and $Y_k$ ($|S_k| = m$ and $|Y_k|=m$) and the older pairs were removed from the collections. We experimented with different L-BFGS memory sizes $m \in \{20,40,80\}$. 

All code is implemented in the Python language using Pytorch, NumPy, and SciPy libraries, and it is available at \texttt{http://rafati.net/quasi-newton-rl}. 

\subsection{Results and Discussions}
\label{sec:discussions}
The average of the maximum game scores is reported in Figure \ref{fig:train-mean-std} (a). The error bar in Figure \ref{fig:train-mean-std} (a) is the standard deviation for the simulations with different batch size, $b \in \{512,1024,2048,4096\}$, and different L-BFGS memory size, $m \in \{20,40,80\}$, for each ATARI game (total of 12 simulations per each task). All simulations regardless of the batch size, $b$, and the memory size, $m$, exhibited robust learning. The average training time for each task, along with the empirical loss values, $\L(w_k)$, is shown in Figure \ref{fig:train-mean-std} (b). 

The Coefficient of Variation (C.V.) for the test scores was about $10\%$ for each ATARI task. (The coefficient of variation is defined as the standard deviation divided by the mean). We did not find a correlation between the test scores and the different batch sizes, $b$, or the different L-BFGS memory sizes, $m$. The coefficient of variation for the training times was about $50\%$ for each ATARI task. Hence, we did not find a strong correlation between the training time and the different batch sizes, $b$, or the different L-BFGS memory sizes, $m$. In most of the simulations, the loss for the training time, as shown in Figure \ref{fig:train-mean-std} (b), was very small.
\begin{figure}[hbt!]
	\centering
	\begin{tabular}[t]{cc}
		(a)&(b)\\
		\includegraphics[width=.44\textwidth]{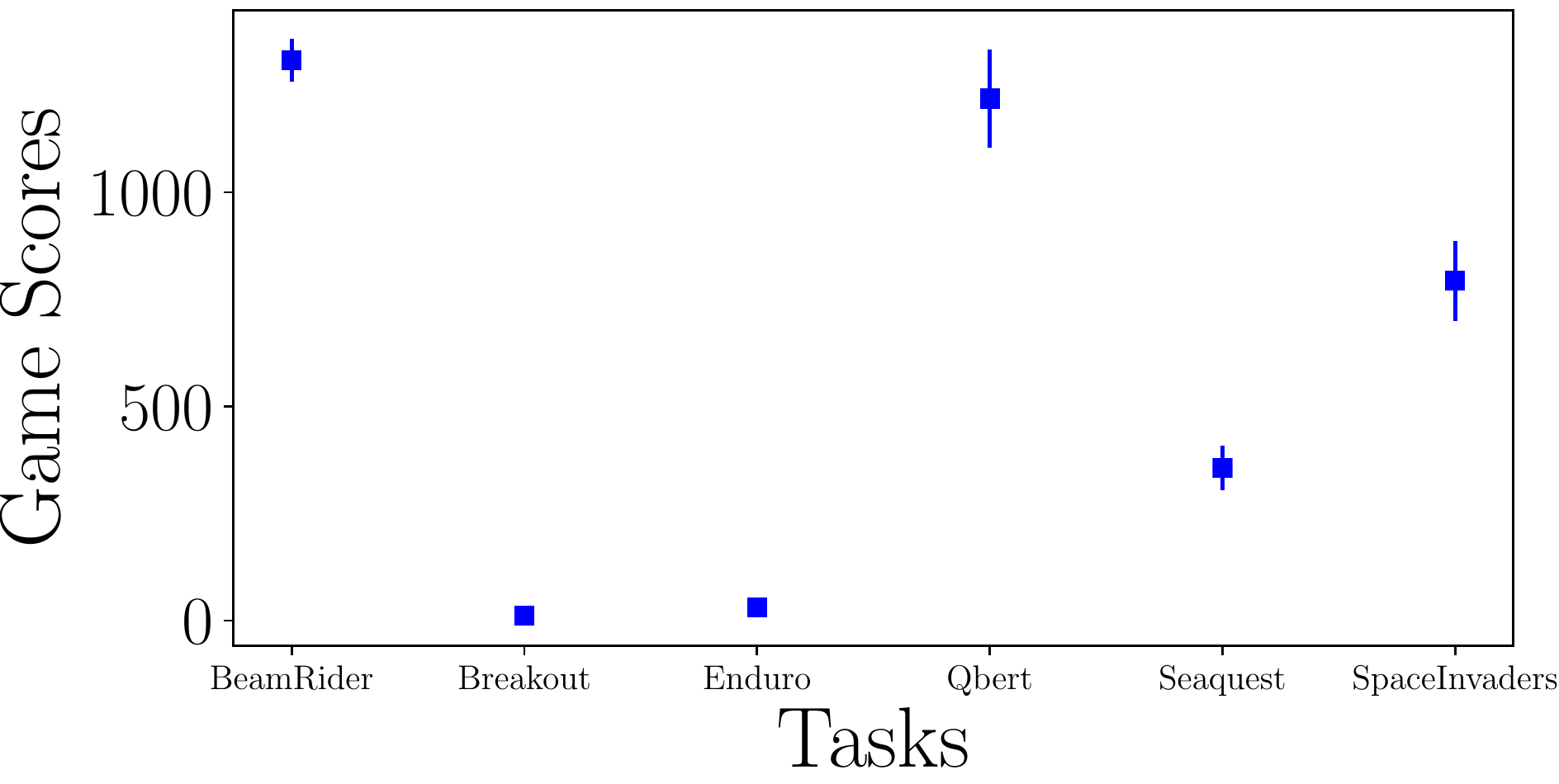}
		&
		\includegraphics[width=.44\textwidth]{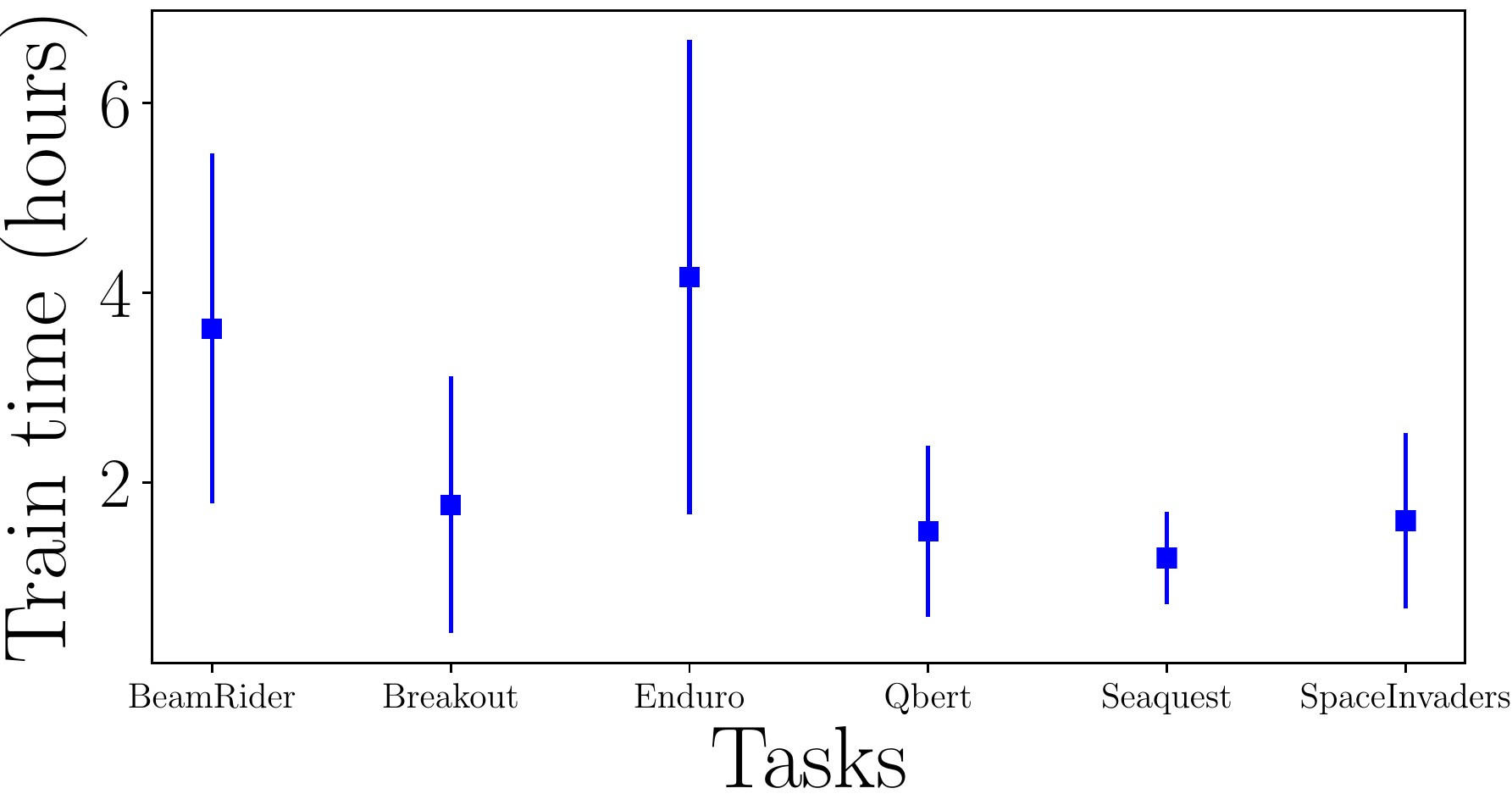}	
	\end{tabular}
	\caption{(a) Test scores (b) Total training time for ATARI games.}
	\label{fig:train-mean-std}
\end{figure}

The test scores and the training loss, $\L_k$, for the six ATARI 2600 environments is shown in Figure \ref{fig:score-time-ATARI-Games} using the batch size of $b=2048$ and L-BFGS memory size $m=40$.

\begin{figure*}[hbt!]
	\centering
	\begin{tabular}{ccc} 
		(a) & (b) & (c) \\
		\includegraphics[width=.3\textwidth]{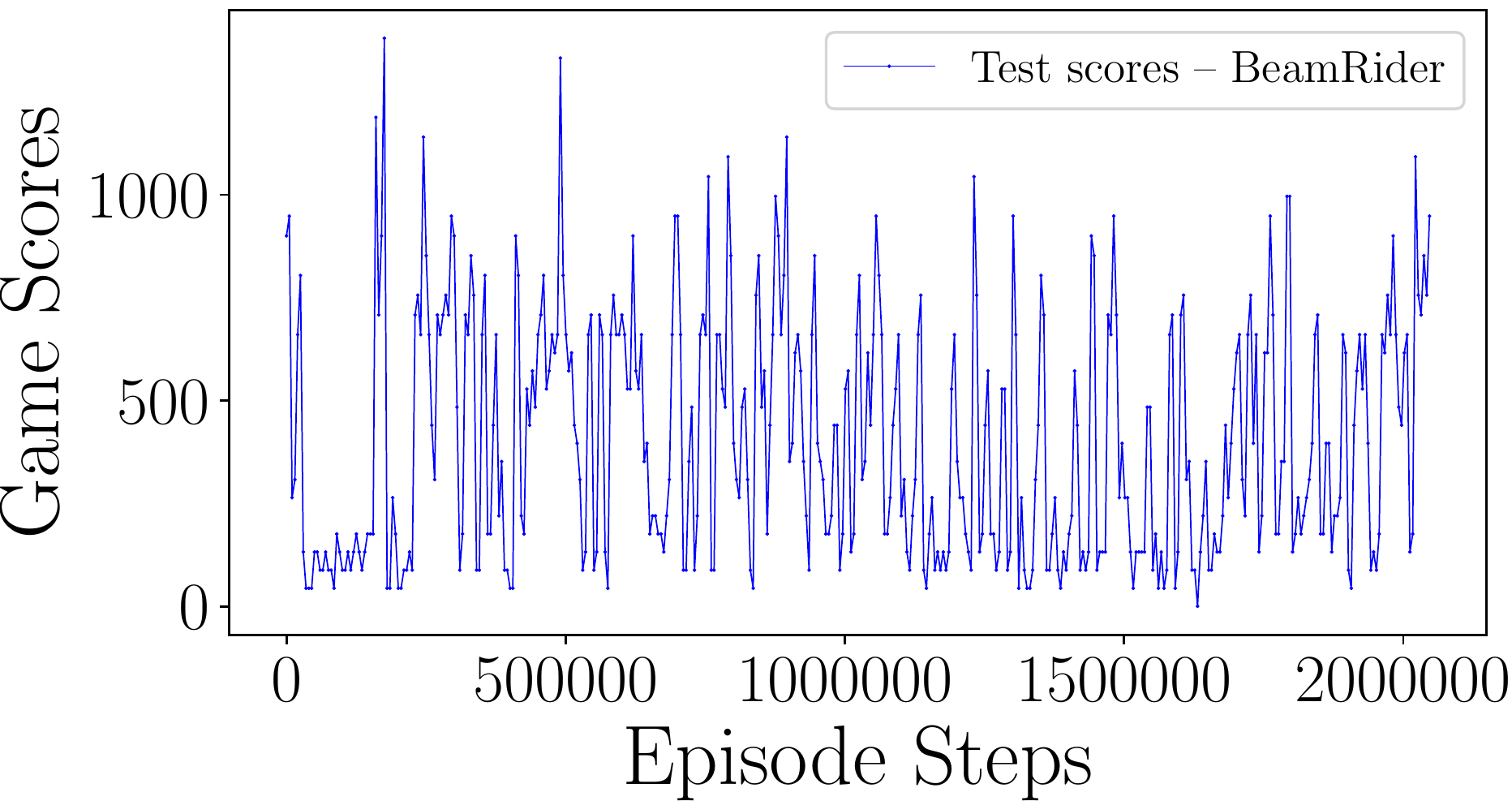} & 
		\includegraphics[width=.3\textwidth]{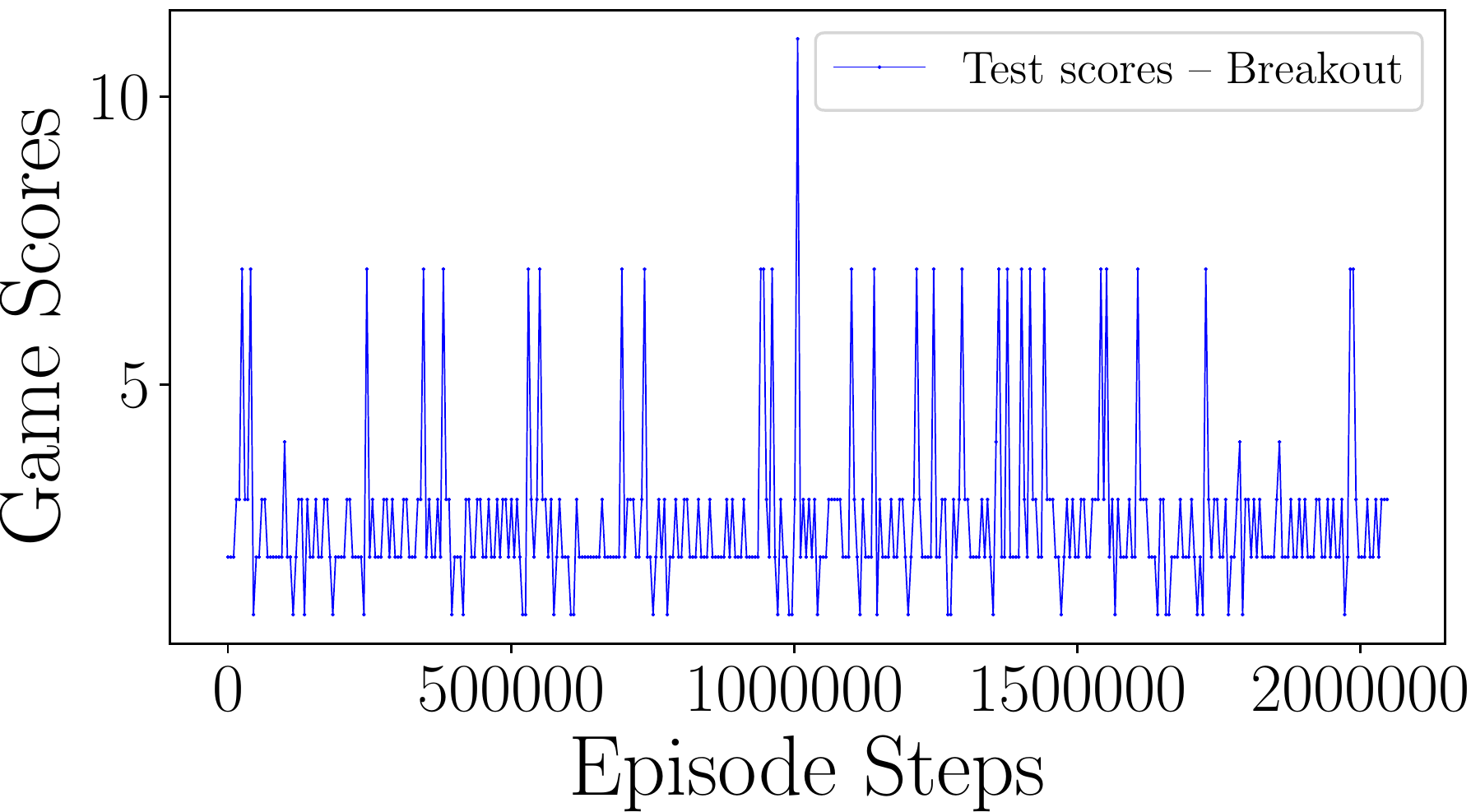} & 	\includegraphics[width=.3\textwidth]{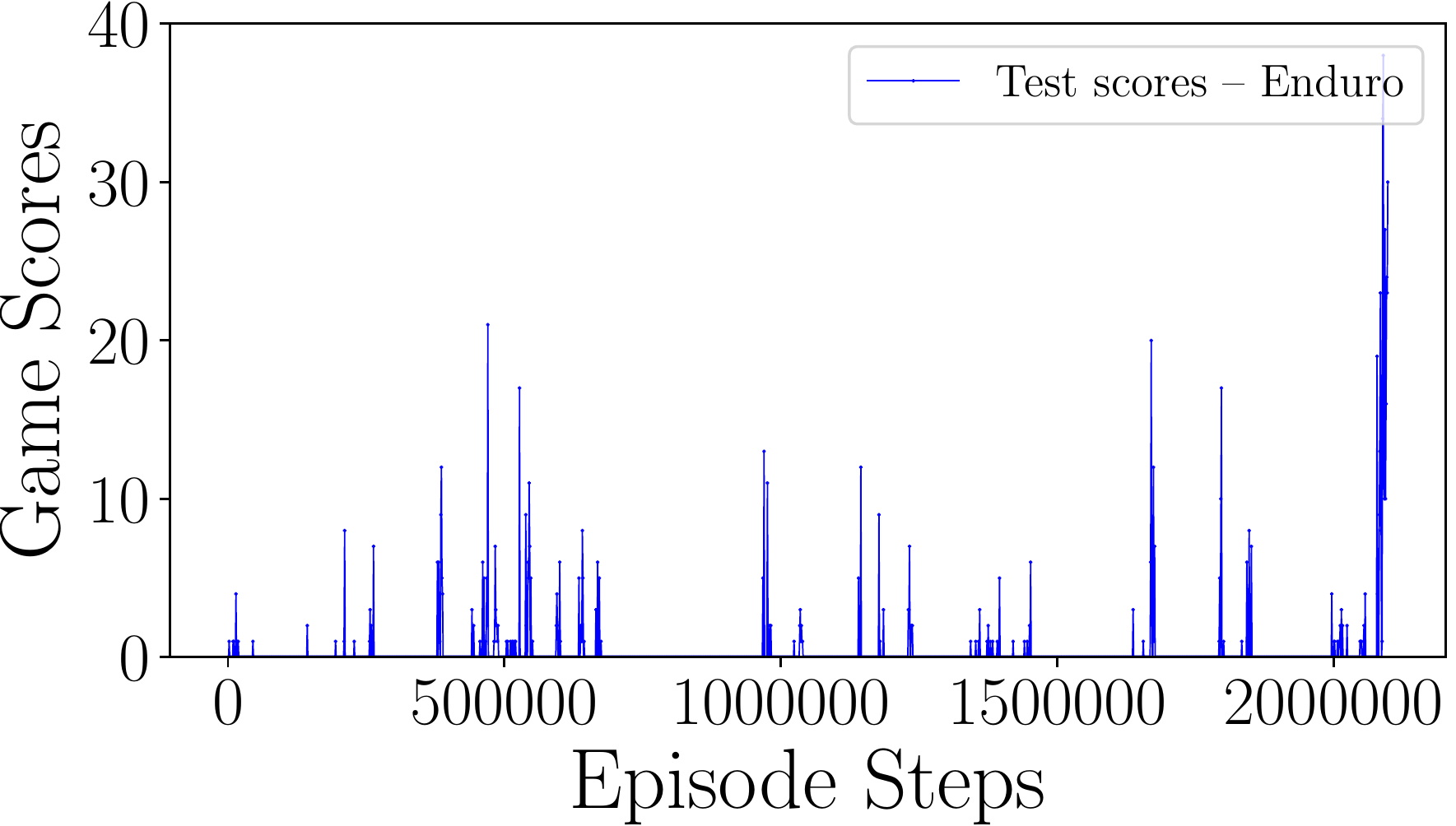} \\
		(d) & (e) & (f) \\
		\includegraphics[width=.3\textwidth]{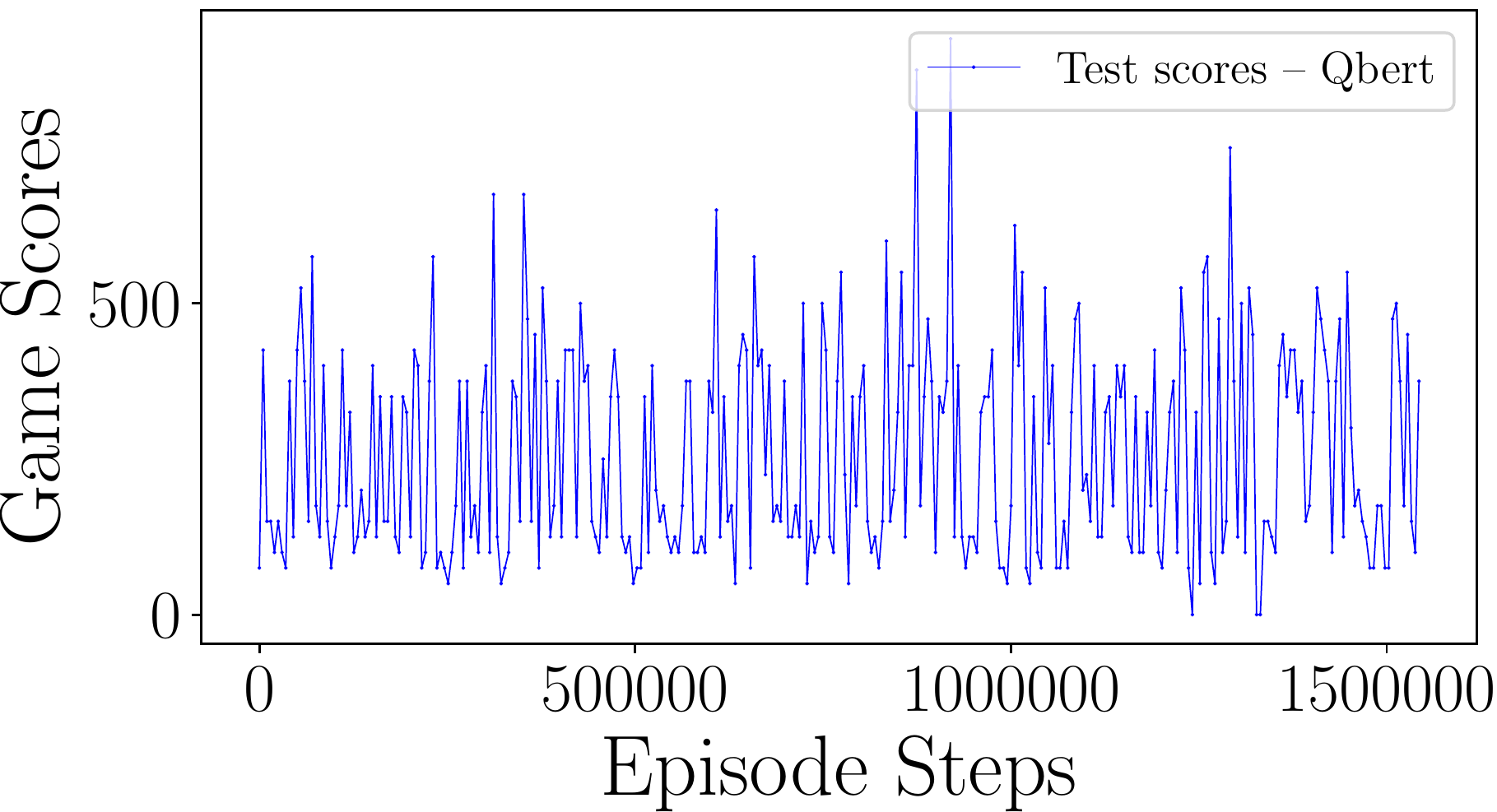} & 
		\includegraphics[width=.3\textwidth]{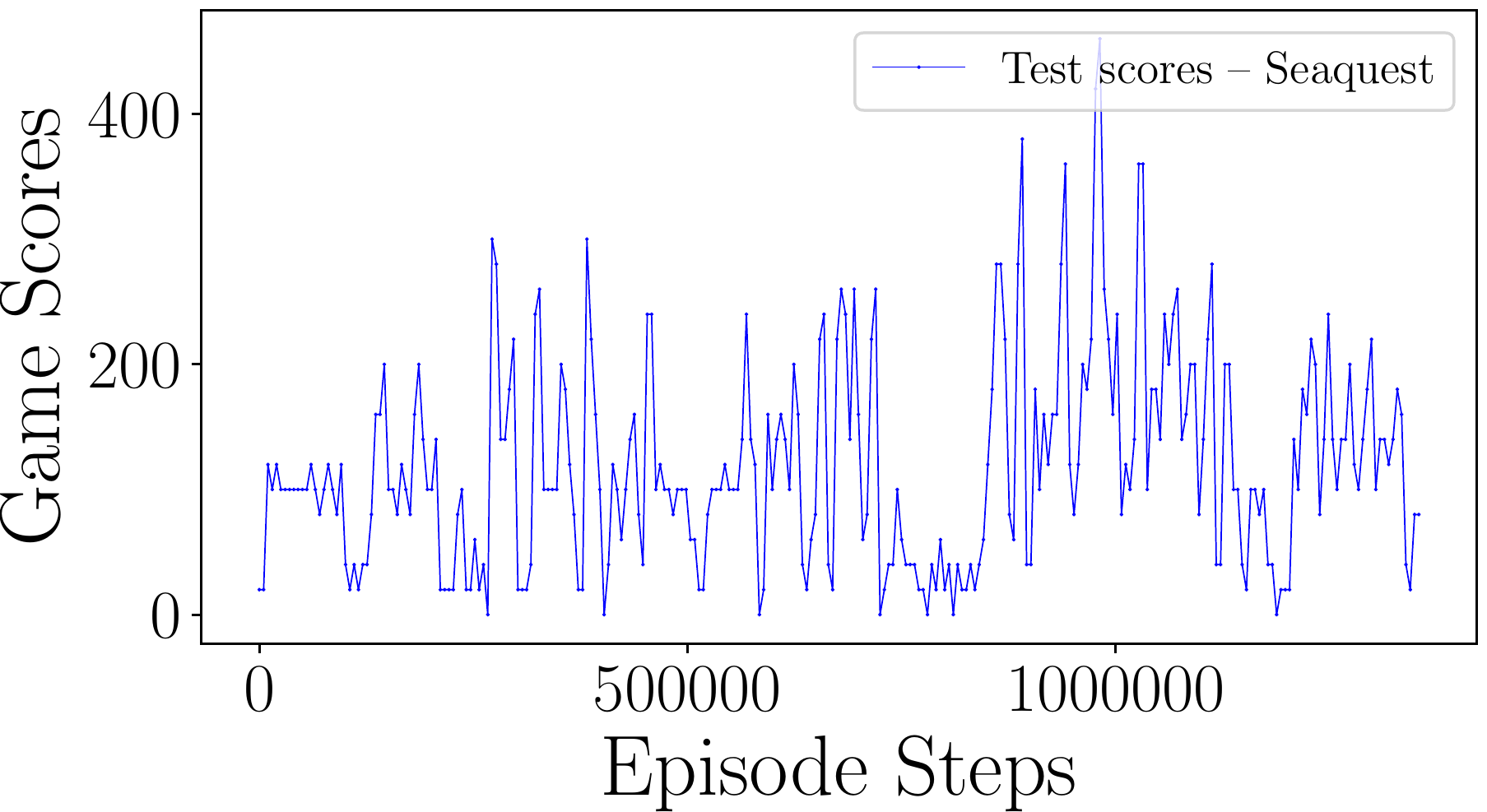} & 	\includegraphics[width=.3\textwidth]{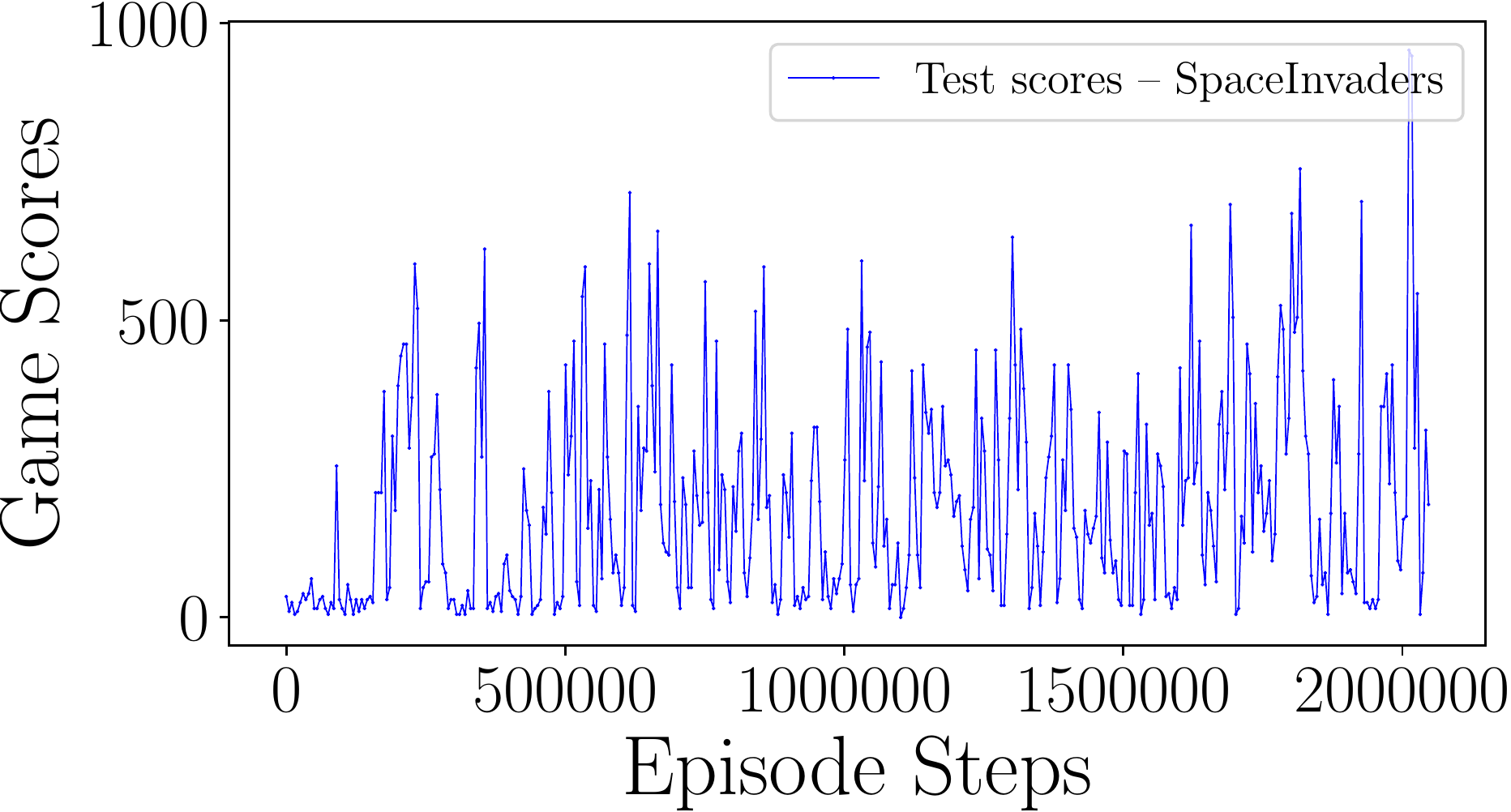} \\
		(g) & (h) & (i) \\ 
		\includegraphics[width=.3\textwidth]{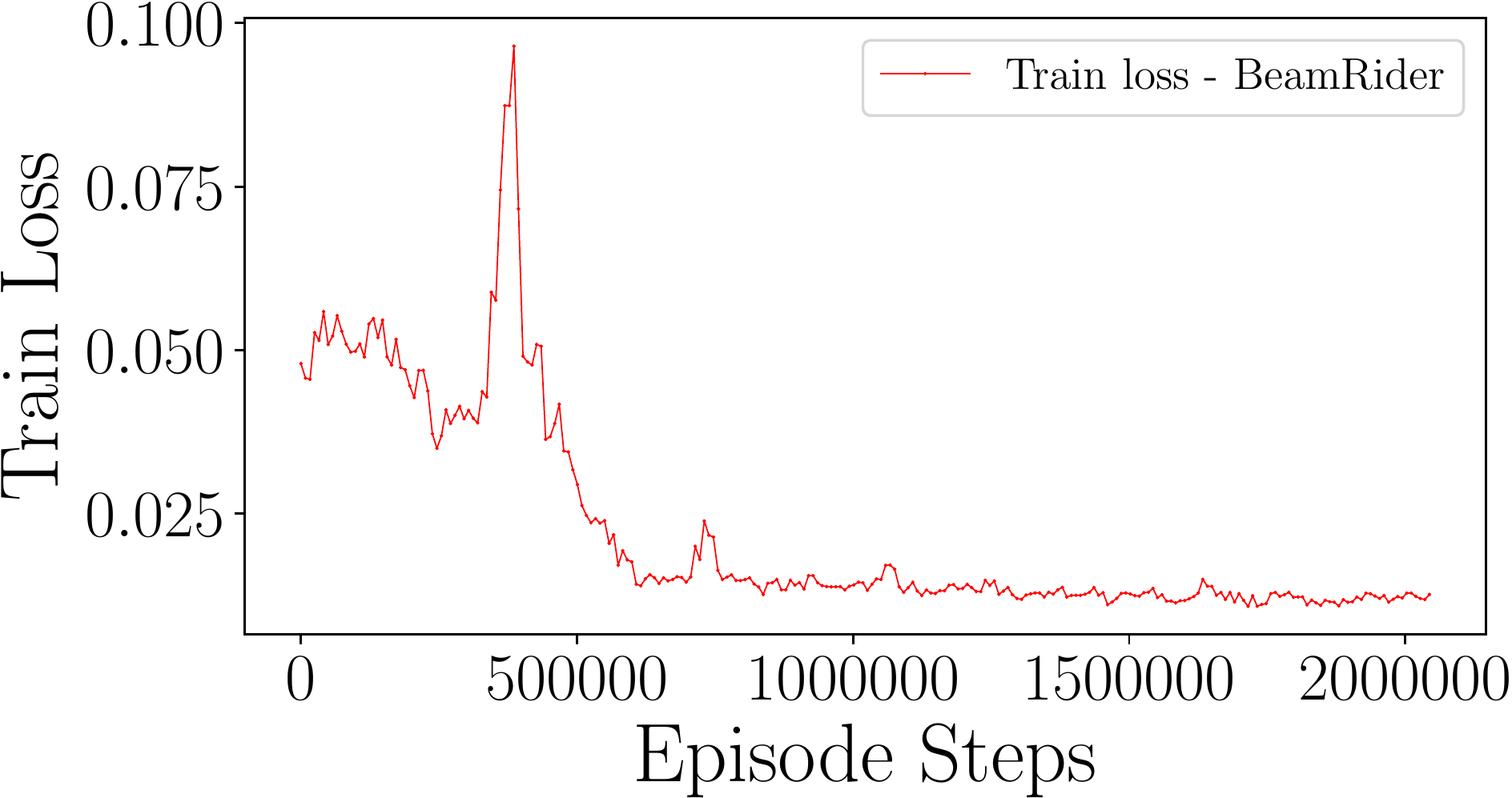} & 
		\includegraphics[width=.3\textwidth]{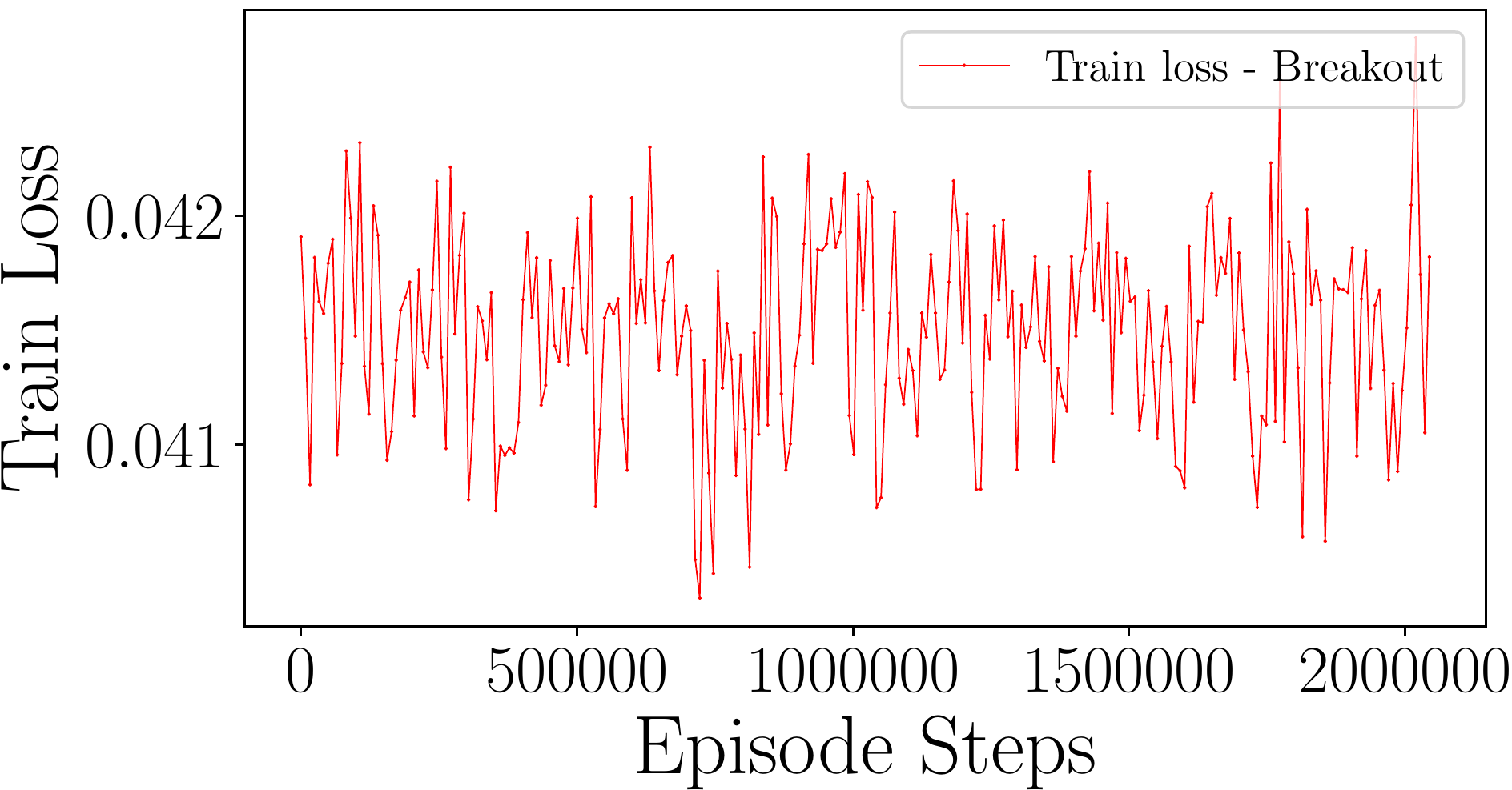} & 	\includegraphics[width=.3\textwidth]{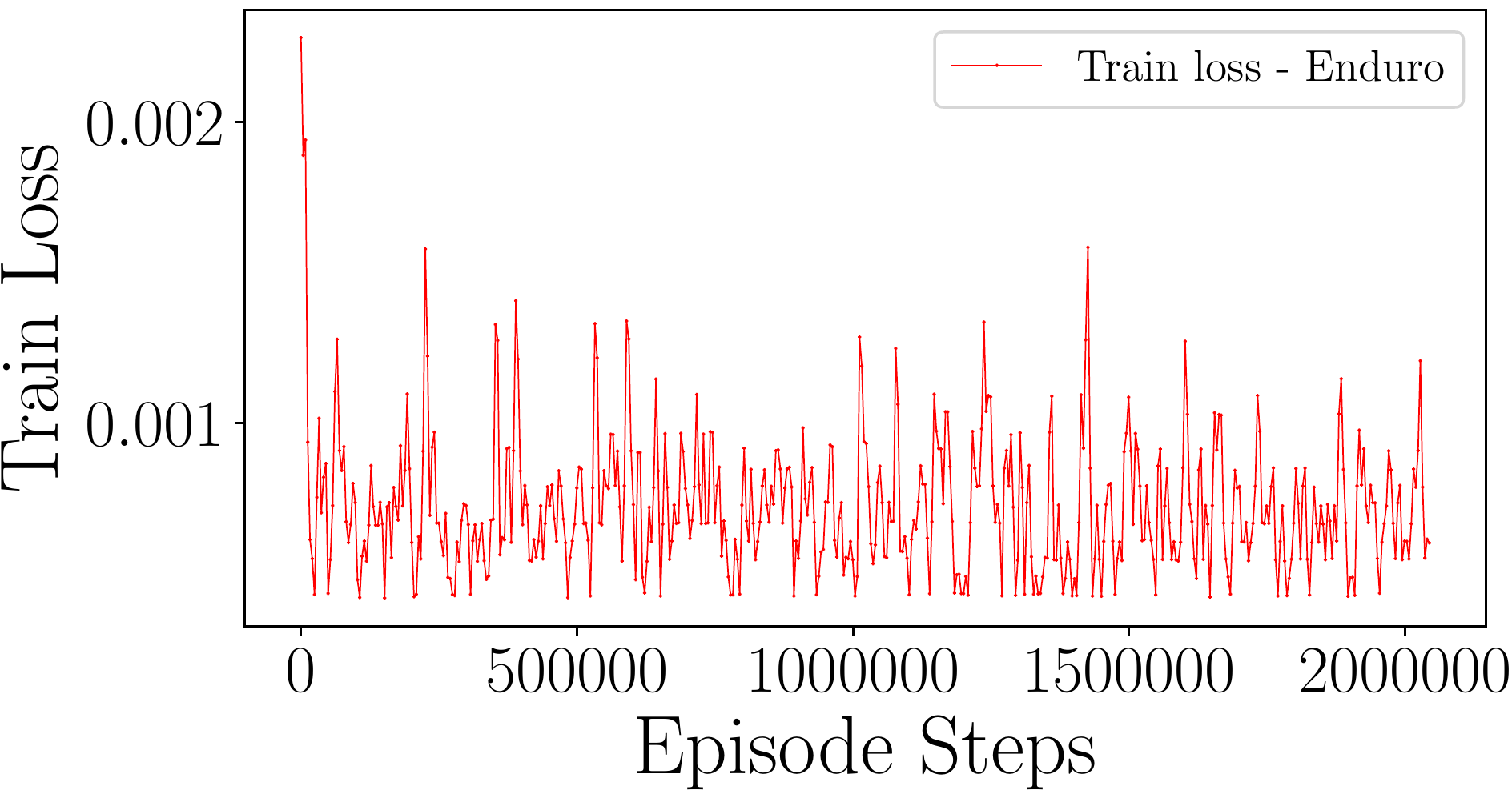} \\
		(j) & (k) & (l)\\
		\includegraphics[width=.3\textwidth]{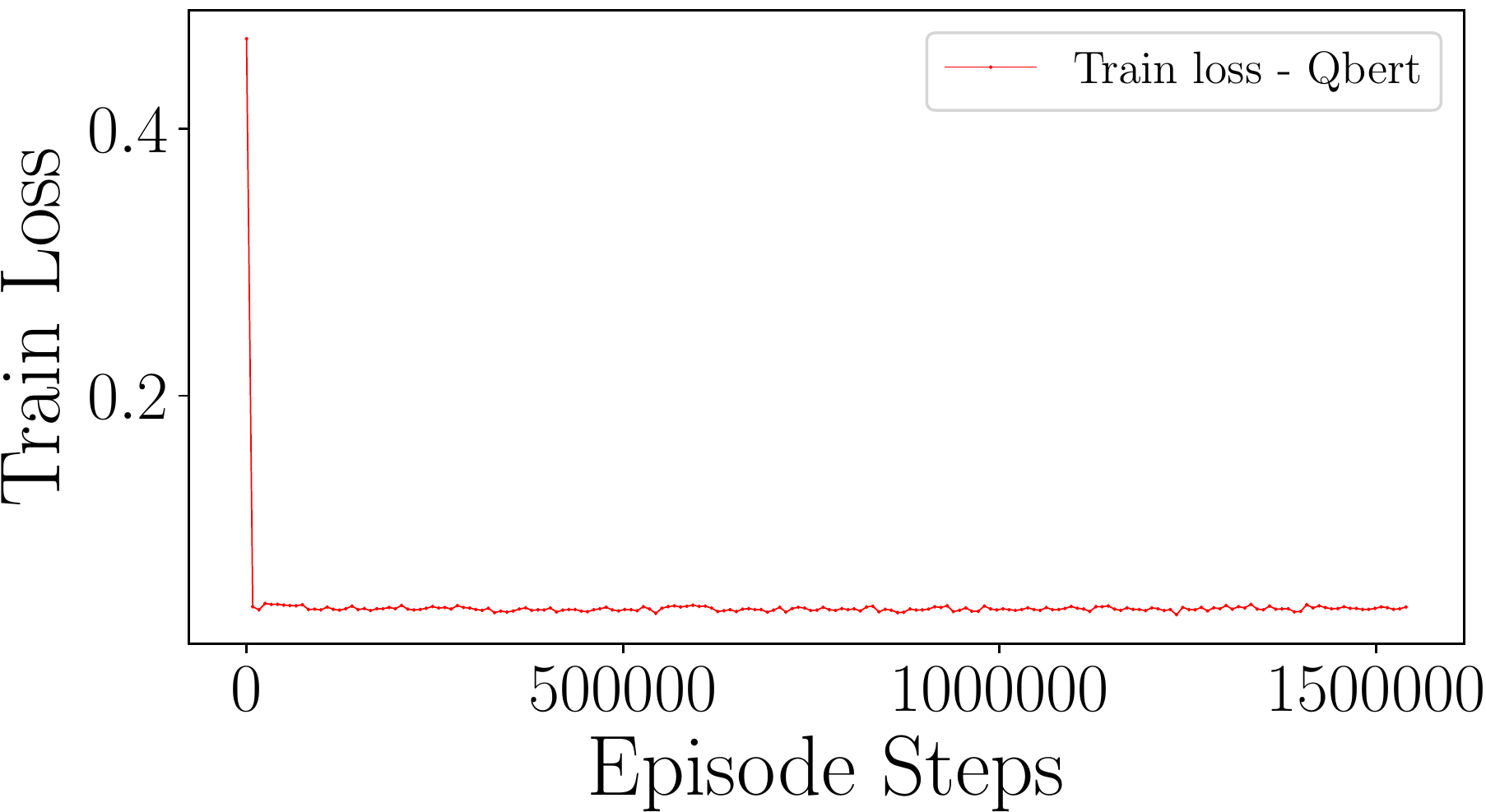} & 
		\includegraphics[width=.3\textwidth]{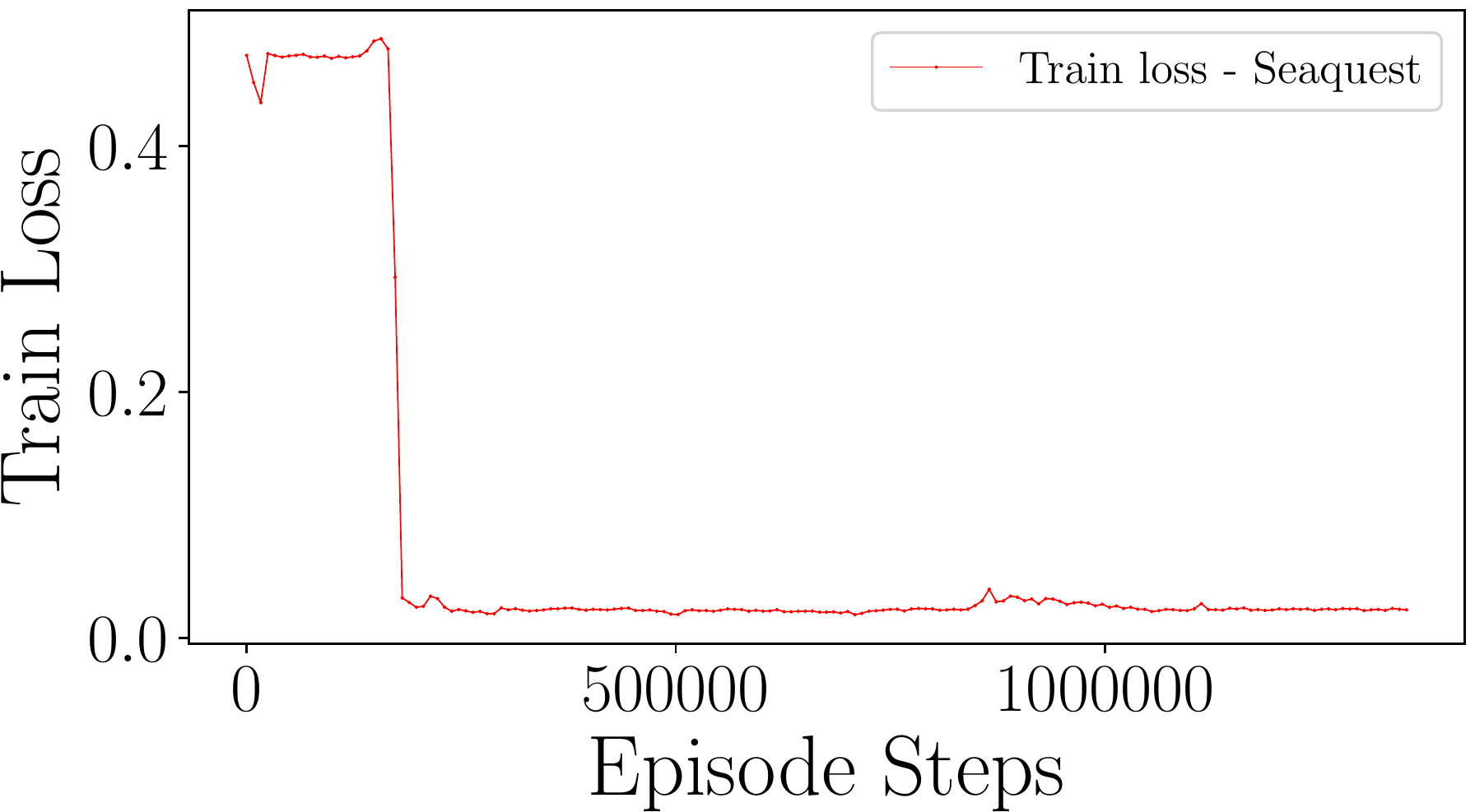} & 	\includegraphics[width=.3\textwidth]{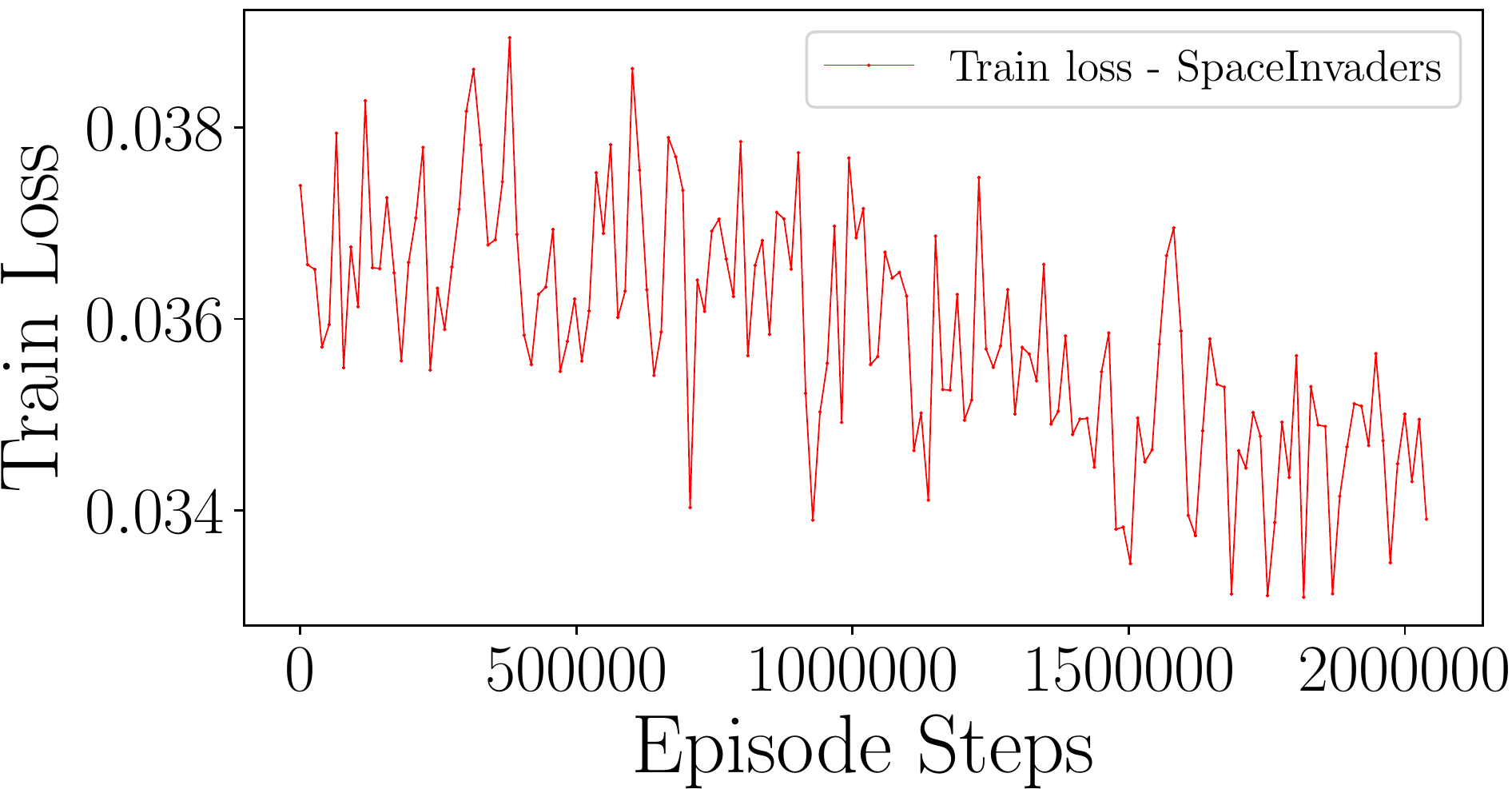} 
	\end{tabular}
	\caption{(a) -- (f) Test scores and (g) -- (l) training loss for six ATARI games --- Beam Rider, Breakout, Enduro, Q*bert, Seaquest, and Space Invaders. The results are form simulations with batch size $b=2048$ and the L-BFGS memory size $m=40$.}
	\label{fig:score-time-ATARI-Games}
\end{figure*}

The results of the Deep L-BFGS Q-Learning algorithm is summarized in Table \ref{t:summary}, which also includes an expert human performance and some recent model-free methods: the Sarsa algorithm \citep{Bellemare:2013:ALE}, the \emph{contingency aware} method from \citep{Bellemare:2012:Contingency}, deep Q-learning \citep{DeepMind:Atari:2013}, and two methods based on policy optimization called Trust Region Policy Optimization (TRPO vine and TRPO single path) \citep{Schulman:2015:TRPO-ATARI} and the Q-learning with the SGD method. Our method outperformed most other methods in the Space Invaders game. Our deep L-BFGS Q-learning method consistently achieved reasonable scores in the other games. Our simulations were only trained on about 2 million Q-learning steps (much less than other methods). DeepMind DQN method outperformed our algorithm on most of the games, except on the Space-Invaders game. 
\begin{table}[hbt!]
	\small
	\centering
	\caption{Best Game Scores for ATARI 2600 Games with different learning methods. Beam Rider, Breakout, Enduro, Q*bert, Seaquest, and Space Invaders.}
	\begin{tabular}[t]{c|cccccc} 
		\hline
		\textbf{Method} &  \textbf{Beam-Rider} & \textbf{Breakout} &  \textbf{Enduro} & \textbf{Q*bert} & \textbf{Seaquest} & \textbf{Space-Invaders} \\ \hline 
		Random & 354 & 1.2 &0 &157 & 110 & 179 \\ \hline 
		Human & 7456 & 31 & 368 & 18900 & 28010 & 3690  \\ \hline 
		Sarsa
		\citep{Bellemare:2013:ALE}&
		996 & 5.2 &
		129 &
		614 &
		665 &
		271 \\ \hline 
		Contingency \citep{Bellemare:2012:Contingency} &
		1743 & 6 &
		159 &
		960 &
		723
		& 268	\\ \hline
		HNeat Pixel \citep{Hausknecht:2014:HNeat-ATARI}&
		1332 & 4
		& 91
		& 1325 &
		800 &
		1145 \\ \hline
		DQN \citep{DeepMind:Atari:2013} & 4092 &168 & 470 & 1952 &1705 & 581 \\ \hline  
		TRPO, Single path \citep{Schulman:2015:TRPO-ATARI} & 1425 & 10 & 534 & 1973 & 1908 &568 \\ \hline 
		TRPO, Vine \citep{Schulman:2015:TRPO-ATARI} & 859 & 34 & 431 & 7732 & 7788 & 450\\ \hline
		SGD ($\alpha = 0.01$) & 804 &13 & 2 & 1325 & 420 & 735 \\ \hline 
		SGD ($\alpha = 0.00001$) & 1092 &14 & 1 & 1300 & 380 & 975 \\ \hline 
		Our method & 1380 & 18 & 49 &1525 & 600 & 955 \\\hline
	\end{tabular}
	\label{t:summary}
\end{table}

The training time for our simulations were on the order of 3 hours (4 hours for Beam-Rider, 2 hours for Breakout, 4 hours for Enduro, 2 hours for Q*bert, 1 hour for Seaquest, and 2 hours for Space-Invaders). Our method outperformed all other methods on the computational time. For example, 500 iterations of the TRPO  algorithm took about 30 hours \citep{Schulman:2015:TRPO-ATARI}. We also compared our method with the SGD method. For each task, we trained the Q-learning algorithm using the SGD optimization method for two million Q-learning training steps. We examined two different learning rates: a relatively large learning rate, $\alpha=0.01$, and a very small learning rate, $\alpha=0.00001$. The other parameters were adopted from \cite{DeepMind:Nature:2015}. The game scores with our method outperformed the SGD method in most of the simulations (11 out of 12 times) (see Table \ref{t:summary}). Although the computation time per iteration of the SGD update is lower than our method, but the total training time of the SGD method is much slower than our method due to the higher frequency of the parameter updates in the SGD method as opposed to our L-BFGS line-search method. See Table \ref{t:time} for the results of the training time for each task, using different optimization methods, L-BFGS and SGD. 
\begin{table}[t!]
	\small
	\centering
	\caption{Average training time for ATARI 2600 games with different learning methods (in hours). Beam Rider, Breakout, Enduro, Q*bert, Seaquest, and Space Invaders.}
	\begin{tabular}[t]{c|cccccc} 
		\hline
		\textbf{Method} &  \textbf{Beam-Rider} & \textbf{Breakout} &  \textbf{Enduro} & \textbf{Q*bert} & \textbf{Seaquest} & \textbf{Space-Invaders} \\ \hline 
		SGD ($\alpha=0.01$) & 4 & 2 & 8 & 8 & 8 & 1 \\ \hline
		SGD ($\alpha=0.00001$) & 7 & 11 & 7 & 6 & 8 & 6 \\ \hline  
		Our method & 4 & 2 & 4 & 2 & 1 & 1 \\\hline
	\end{tabular}
	\label{t:time}
\end{table}

\section{Conclusions}
\label{sec:conclusion}
In this chapter, we implemented an optimization method based on the limited-memory quasi-Newton method known as L-BFGS as an alternative to the gradient descent methods typically used to train deep neural networks. We considered both line-search and trust-region frameworks. The contribution of this research is an algorithm known as TRMinATR which minimizes the cost function of the neural network by efficiently solving a sequence of trust-region subproblems using low-rank updates to Hessian approximations. The benefit of the method is that the algorithm is free from the constraints of data specific parameters seen in traditionally used methods. TRMinATR also improves on the computational efficiency of a similar line search implementation.
Furthermore, we proposed and implemented a novel optimization method based on line search limited-memory BFGS for the deep reinforcement learning framework. We tested our method on six classic ATARI 2600 games. The L-BFGS method attempts to approximate the Hessian matrix by constructing  positive definite matrices with low-rank updates. Due to the nonconvex and nonlinear loss functions arising in deep reinforcement learning, our numerical experiments show that using the curvature information in computing the search direction leads to a more robust convergence when compared to the SGD results. Our proposed deep L-BFGS Q-Learning method is designed to be efficient for parallel computations on GPUs. Our method is much faster than the existing methods in the literature, and it is memory efficient since it does not need to store a large experience replay memory.
Since our proposed limited-memory quasi-Newton optimization methods rely only on first-order gradients, they can be efficiently scaled and employed for larger scale supervised learning, unsupervised learning, and reinforcement learning applications. The overall enhanced performance of our proposed optimization methods on deep learning applications can be attributed to the robust convergence properties, fast training time, and better generalization characteristics of these optimization methods.

\bibliography{draft.bib}
\bibliographystyle{unsrt}

\end{document}